\documentclass{article}

\PassOptionsToPackage{numbers, compress}{natbib}
\PassOptionsToPackage{prologue,dvipsnames}{xcolor}

\usepackage[table]{xcolor}
 \usepackage{xcolor}         

 \usepackage[main, final]{neurips_2026}

\usepackage[utf8]{inputenc} 
\usepackage[T1]{fontenc}    
\usepackage{hyperref}       
\usepackage{url}            
\usepackage{booktabs}       
\usepackage{amsfonts}       
\usepackage{nicefrac}       
\usepackage{microtype}      
\usepackage{amsmath,amssymb,amsthm}
\usepackage{mathtools}
\usepackage{graphicx}
\usepackage{booktabs}
\usepackage{hyperref}
\usepackage{url}
\usepackage{xcolor}
\usepackage{algorithm}
\usepackage{algpseudocode}
\usepackage{thmtools}         
 
\newcommand{\vx}{\mathbf{x}}
\newcommand{\vy}{\mathbf{y}}
\newcommand{\vz}{\mathbf{z}}
\newcommand{\vc}{\mathbf{c}}

\newcommand{\vxi}{\boldsymbol{\xi}}
\newcommand{\Ac}{\mathcal{A}}
\newcommand{\eg}{\emph{e.g.}}
\newcommand{\ie}{\emph{i.e.}}
\newcommand{\vxinit}{\mathbf{x}_\text{init}}
\newcommand{\vzinit}{\mathbf{z}_\text{init}}
\newcommand{\vfield}{\mathbf{v}}
\newcommand{\KV}{\text{KV}}

\theoremstyle{plain}

\theoremstyle{definition}

\theoremstyle{remark}

\definecolor{custompink}{HTML}{E5A4CB}
\definecolor{customblue}{HTML}{85C1E9}

\definecolor{custompinktext}{HTML}{D76FAD}
\definecolor{custombluetext}{HTML}{2790D7}

\title{Accelerating Video Inverse Problem Solvers \\ with Autoregressive Diffusion Models}

\author{
    Taesung Kwon\textsuperscript{\textmd{1*}}
    \qquad Jonghyun Park\textsuperscript{\textmd{1*}}
    \qquad Hyungjin Chung\textsuperscript{\textmd{2\dag}}
    \qquad Jong Chul Ye\textsuperscript{\textmd{1\dag}} \\[3pt]
    {\normalfont \textsuperscript{1}KAIST \qquad \textsuperscript{2}Korea University}
}

\begin{document}

\maketitle

\begin{figure}[h]
    \centering
    \vspace{-1.0em}
    \includegraphics[width=\linewidth]{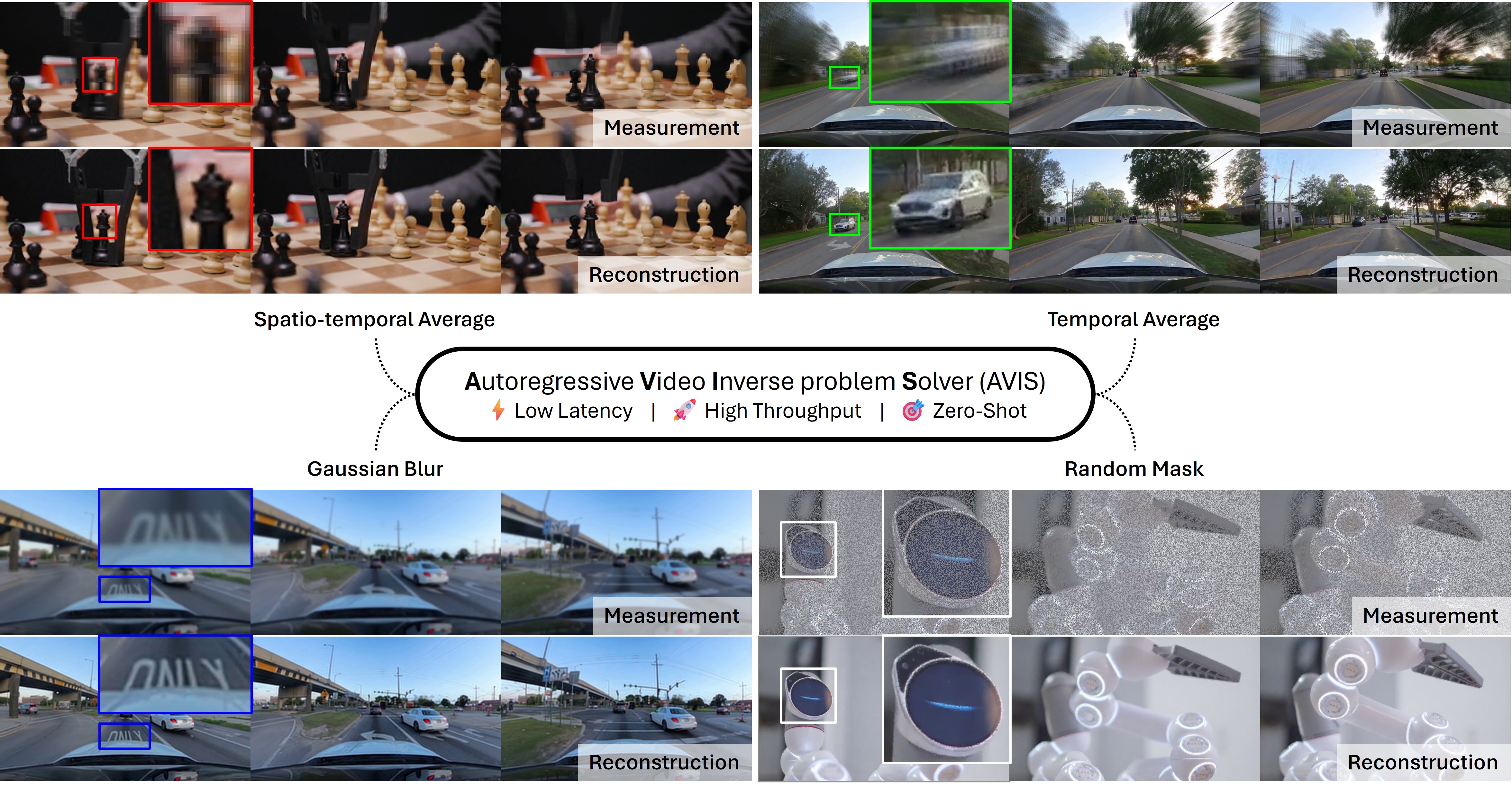}
    \vspace{-1.5em}
    \caption{The AVIS framework leverages autoregressive video diffusion models to restore videos in a streaming manner, naturally eliminating latency bottlenecks. Examples show diverse video inverse problems that AVIS can solve in a zero-shot manner.}
    \label{fig:teaser}
\end{figure}

\definecolor{avispink}{HTML}{D76FAD}
\let\thefootnote\relax\footnotetext{%
  \textsuperscript{*}Equal contribution.\quad
  \textsuperscript{\dag}Corresponding authors.\hfill
  Project page and code are available \href{https://avis-project.github.io/}{\textcolor{avispink}{here}}.
}

\begin{abstract}
Diffusion models provide powerful priors for zero-shot video inverse problems, but their real-time deployment is hindered by two inefficiencies: high initial latency caused by holistic video restoration, and low throughput resulting from multiple VAE passes to enforce measurement consistency in pixel space.
To overcome these limitations, we propose Autoregressive Video Inverse problem Solver (AVIS).
The AVIS framework leverages autoregressive video diffusion models to restore videos in a streaming manner, naturally eliminating latency bottlenecks.
Specifically, AVIS initializes reverse diffusion with a measurement-consistent estimate, reducing the required sampling steps.
Compared to leading non-autoregressive solvers, AVIS drastically reduces initial latency from 114s to 4s and increases throughput from 0.71 to 1.18 FPS while achieving superior restoration quality. 
We further introduce a highly accelerated variant, dubbed AVIS \emph{Flash}, that enforces measurement consistency solely on the \emph{first} chunk.
AVIS \emph{Flash} substantially boosts throughput to 5.91 FPS on a single RTX 4090 GPU while maintaining competitive performance and achieving a favorable efficiency–performance trade-off, paving the way toward real-time deployment.
\end{abstract}

\section{Introduction}
\label{sec:intro}
Video inverse problems aim to reconstruct a clean video $\vx$ from a degraded measurement $\vy$ (\eg, low-resolution, masked, or blurred). The degradation process is mathematically formulated as
\begin{equation}
  \vy = \Ac(\vx) + \boldsymbol{n},
\end{equation}
where $\Ac$ is a degradation operator and $\boldsymbol{n} \sim \mathcal{N}(0,\sigma_n^2 I)$ denotes measurement noise.
These video restoration tasks are ill-posed, as a single measurement $\vy$ can correspond to many possible clean videos $\vx$.
To resolve this ambiguity, diffusion and flow-based models~\cite{ho2020denoising, song2020score, dhariwal2021diffusion, rombach2022high, ho2022video, podell2024sdxl, liu2023flow, lipman2023flow, esser2024scaling} have recently emerged as powerful priors for solving video inverse problems in a zero-shot manner~\cite{chang2024how, daras2024warped, kwon2025solving, kwon2025vision, kwon2025video, spagnoletti2026lvtino}, producing results that are both perceptually realistic and faithful to the measurements.
Despite their successful restoration performance, these solvers remain far from real-time deployment due to two inefficiencies:
\begin{enumerate}
    \item \textbf{High initial latency:}
    Recent Diffusion model-based Video Inverse problem Solvers (DVIS)~\cite{daras2024warped, kwon2025solving, kwon2025vision, kwon2025video, spagnoletti2026lvtino} restore the entire video in a holistic manner, sampling the entire video sequence simultaneously.
    Consequently, the first frame cannot be displayed until the last frame has been restored, introducing latency equal to the time required to restore all frames.
    \item \textbf{Low throughput:} Modern diffusion and flow-based video models~\cite{kong2024hunyuanvideo, hacohen2024ltx, wan2025wan, agarwal2025cosmos, ali2025world} operate in the latent space of a variational autoencoder (VAE)~\cite{kingma2013auto} for computational efficiency. Recent DVIS based on such models~\cite{kwon2025vision, spagnoletti2026lvtino} require multiple VAE passes during restoration, restricting throughput to below 1 FPS. While some works~\cite{raphaeli2025silo, hong2025inversecrafter} bypass this with auxiliary latent operators, they require degradation-specific training, sacrificing zero-shot generality.
\end{enumerate}

Here, we identify a promising yet underexplored direction toward the efficient deployment of DVIS: harnessing the potential of autoregressive (AR) video diffusion models~\cite{chen2024diffusion, yin2025slow, huang2025self}.
Unlike non-autoregressive solvers that process the entire video jointly (Figure~\ref{fig:intro}(a)), autoregressive solvers restore videos in a streaming manner, naturally eliminating the latency bottleneck.
However, simply adopting an AR backbone does not automatically resolve the issue of low throughput, as it still requires multiple VAE passes for measurement consistency updates.

To overcome this limitation, we propose \textbf{A}utoregressive \textbf{V}ideo \textbf{I}nverse problem \textbf{S}olver (\textbf{AVIS}), which accelerates restoration by initializing the reverse diffusion process with a measurement-consistent estimate, thereby reducing the required sampling steps.
We draw inspiration from CCDF~\cite{chung2022come}, which demonstrates that initializing the reverse diffusion process from a coarse estimate, predicted by an auxiliary pretrained restorer, exploits stochastic contraction to significantly reduce the required sampling steps.
Unlike CCDF, AVIS does not rely on an auxiliary pretrained restorer; instead, it obtains an initial restoration by directly minimizing the measurement consistency objective.
{By utilizing this estimate to initialize the reverse diffusion, 
AVIS begins the restoration process from a better starting point while enforcing measurement consistency for \emph{every} video chunk (Figure~\ref{fig:intro}(b)).}

To push efficiency further, we introduce a highly accelerated variant, \textbf{AVIS \emph{Flash}}.
Unlike AVIS, which applies iterative measurement guidance to every chunk, AVIS \emph{Flash} enforces measurement consistency only on the \emph{first} video chunk.
Interestingly, we observe that this is sufficient to restore the entire sequence (Figure~\ref{fig:intro}(c)).
Once the first chunk is fully restored and accurately grounded to the measurement, subsequent chunks can be generated through autoregressive propagation from this corrected prefix, completely bypassing explicit measurement guidance.
This eliminates iterative VAE passes for subsequent chunks, accelerating restoration by 5$\times$ while preserving zero-shot performance.

Compared to the leading non-autoregressive solver~\cite{spagnoletti2026lvtino}, AVIS reduces initial latency from 114s to 4s and improves throughput from 0.71 to 1.18 FPS while achieving superior restoration quality.
Pushing efficiency further, AVIS \emph{Flash} substantially boosts throughput to 5.91 FPS on a single RTX 4090 GPU while maintaining competitive performance, paving the way toward real-time deployment.

Our main contributions can be summarized as follows:

\begin{itemize}
    \item \textbf{AVIS.} We investigate the potential of adopting autoregressive video diffusion models for zero-shot video restoration. AVIS accelerates restoration by initializing the reverse diffusion process with a measurement-consistent estimate, reducing the required sampling steps while enforcing measurement consistency for every video chunk.
    \item \textbf{AVIS \emph{Flash}.} We introduce a highly accelerated variant of AVIS that retains the initialization while enforcing measurement consistency solely on the first video chunk. This substantially increases throughput, paving the way for real-time video inverse problem solvers.
\end{itemize}

\begin{figure}[t]
    \centering
    \includegraphics[width=\linewidth]{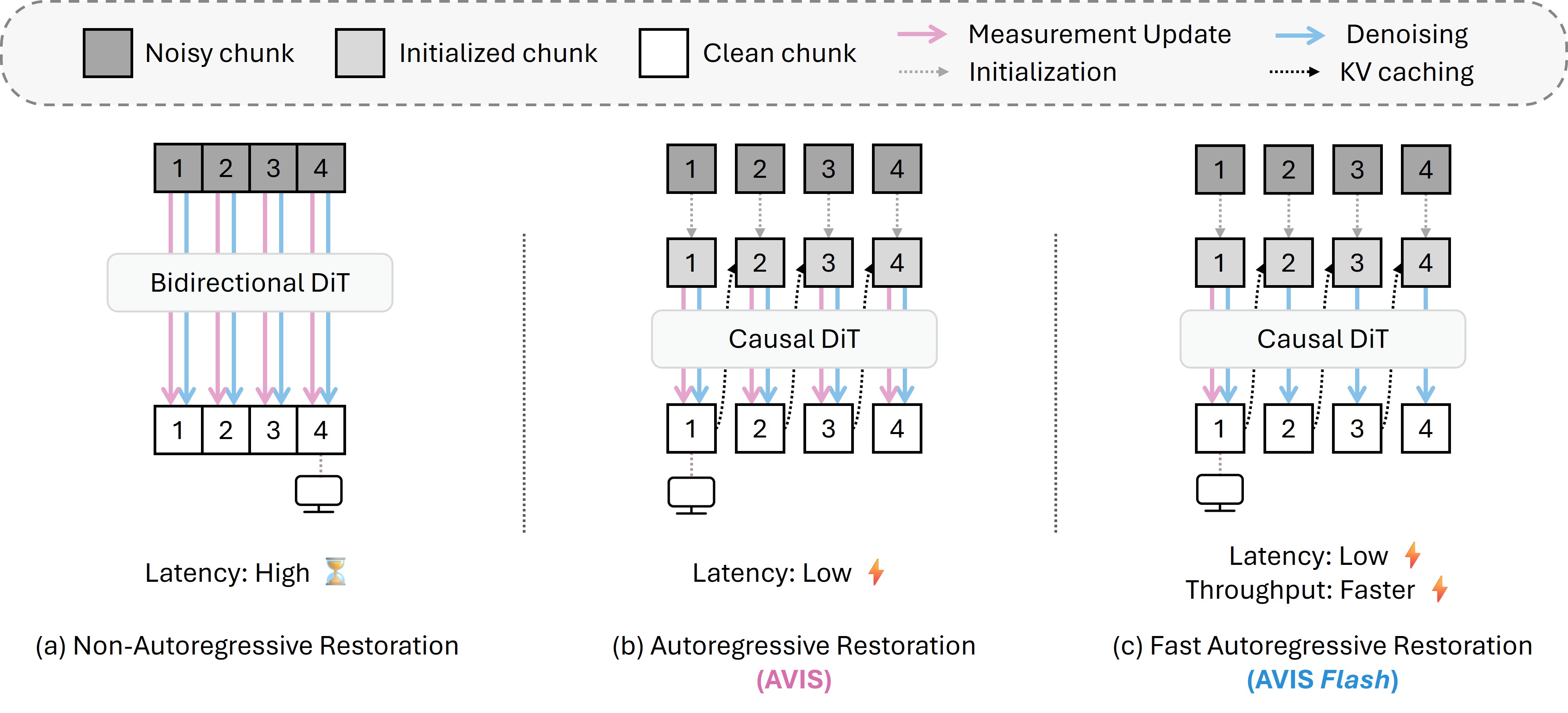}
    \vspace{-2em}
    \caption{\textbf{Overview of our proposed AVIS {and AVIS \emph{Flash} framework.}}
    (a) Non-autoregressive restoration processes the entire video holistically, suffering from high initial latency.
    (b) \textcolor{custompinktext}{AVIS} restores videos in a streaming manner and reduces sampling steps via measurement-consistent initialization while enforcing measurement updates for every video chunk.
    (c) \textcolor{custombluetext}{AVIS \emph{Flash}} retains the same initialization as AVIS but applies measurement updates only to the first video chunk. Subsequent chunks ($n \ge 2$) are restored through autoregressive propagation from the corrected prefix, eliminating iterative VAE passes and dramatically accelerating throughput.}
    \label{fig:intro}
    \vspace{-1em}
\end{figure}

\section{Related Work}
\label{sec:related}
\subsection{Diffusion for Video Inverse Problems}

Diffusion model-based Inverse problem Solvers (DIS)~\cite{chung2023diffusion, song2023pseudoinverseguided, wang2023zeroshot, chung2024decomposed, mardani2024a, song2024solving, rout2024solving, zhang2025improving, kim2025flowdps, park2025flowlps} enable repurposing the diffusion or flow-based models~\cite{ho2020denoising, song2020score, dhariwal2021diffusion, rombach2022high, podell2024sdxl, liu2023flow, lipman2023flow, esser2024scaling} originally trained to model the data distribution $p(\vx)$ as powerful plug-and-play priors.
By guiding the sampling trajectory toward the posterior $p(\vx|\vy)$, these methods achieve strong zero-shot performance across a wide range of image restoration tasks, producing solutions that are perceptually natural and consistent with the measurement $\vy$.

Motivated by this success, Diffusion model-based Video Inverse problem Solvers (DVIS)~\cite{daras2024warped, kwon2025solving, kwon2025vision, kwon2025video, spagnoletti2026lvtino} extend this paradigm to the video modality.
One line of work~\cite{daras2024warped, kwon2025solving, kwon2025vision} repurposes image diffusion priors~\cite{dhariwal2021diffusion, podell2024sdxl} for video restoration by introducing explicit temporal conditioning.
Warping-based approaches~\cite{chang2024how, daras2024warped} explicitly control temporal consistency using optical flow~\cite{teed2020raft}.
An alternative approach enforces temporal consistency without optical flow through batch-consistent sampling~\cite{kwon2025solving, kwon2025vision}, which efficiently synchronizes stochastic components across frames.
More recently, methods leveraging native video diffusion priors~\cite{kwon2025video, spagnoletti2026lvtino} have shown an improved ability to capture temporal correlations without requiring explicit temporal conditioning.

As mentioned earlier, all of these solvers restore the video holistically by sampling all frames simultaneously.
As a result, the first frame cannot be displayed until the entire sequence is fully restored, introducing latency equal to the time required to restore all frames. We pinpoint this issue as a critical barrier for real-time deployment and propose a direction toward efficient DVIS.

\subsection{Autoregressive Diffusion Models for Video Generation}

In recent years, both closed- and open-source models~\cite{sora2024, veo2025, kong2024hunyuanvideo, hacohen2024ltx, wan2025wan, agarwal2025cosmos, ali2025world} have demonstrated remarkable progress in foundational video generation.
Currently, the dominant open-source video diffusion models~\cite{kong2024hunyuanvideo, hacohen2024ltx, wan2025wan, ali2025world} typically adopt diffusion transformers~\cite{peebles2023scalable} with bidirectional attention~\cite{vaswani2017attention} that generate the entire video sequence simultaneously.

Recent works~\cite{chen2024diffusion, yin2025slow, huang2025self, ruhe2024rolling, kim2024fifo, zhang2025framepack, jin2025pyramidal} have explored autoregressive (AR) modeling, which generates the next frame by conditioning on previously generated ones, enabling low-latency generation.
One representative direction is distilling diffusion transformers with bidirectional attention into student transformers with causal attention, allowing support for key-value (KV) caching during inference~\cite{yin2025slow, huang2025self}.
Among these, Self Forcing~\cite{huang2025self} has emerged as a leading autoregressive video diffusion model.
It addresses the train-test gap by conditioning on its own previously generated outputs rather than ground-truth frames during training, preventing error accumulation (\ie, exposure bias) at test time.
Given its effectiveness in closing this gap and its KV caching efficiency, we adopt Self Forcing as the backbone of our framework.

\section{Preliminaries: Autoregressive Video Diffusion Models}
\label{sec:preliminary}
\paragraph{Autoregressive (AR) modeling.}
AR video diffusion models combine sequential generation along the temporal axis with denoising diffusion processes.
Specifically, let $\vz = [\vz^1, \vz^2, \dots, \vz^N]$ be a sequence of $N$ video chunks, where each \emph{chunk} refers to a block corresponding to a short sequence of consecutive latent frames.
Rather than modeling the entire video at once, the model builds the sequence progressively, predicting each new video chunk conditioned on previously generated ones.
Mathematically, the joint distribution of the sequence is factorized sequentially as follows:
\begin{equation}
    p(\vz)=\prod_{n=1}^{N}p(\vz^{n}|\vz^{<n}).
\end{equation}
Each conditional distribution $p(\vz^{n}|\vz^{<n})$ is then parameterized by a diffusion model, where $\vz^n$ is generated by gradually denoising the initial Gaussian noise, conditioned on the past context $\vz^{<n}$, with the initial context $\vz^0=\emptyset$.

\paragraph{Diffusion process.}
To model the conditional distribution $p(\vz^{n}|\vz^{<n})$, we employ the flow matching framework~\cite{liu2023flow, lipman2023flow, esser2024scaling}, which offers a simple and deterministic ordinary differential equation (ODE) formulation for the diffusion process.
Specifically, let $\vz^{n}_0$ denote the clean $n$-th video chunk and $\vz^{n}_1 \sim \mathcal{N}(\mathbf{0}, \mathbf{I})$ be pure Gaussian noise. For a continuous timestep $t \in [0, 1]$, the intermediate noisy state $\vz^{n}_t$ is formulated as a linear interpolation:
\begin{equation}
    \label{eq:interpolation}
    \vz^{n}_t = (1-t)\vz^{n}_0 + t\vz^{n}_1.
\end{equation}
The corresponding ground-truth vector field driving this transformation is defined as $\vfield^{n}_t = \vz^{n}_1 - \vz^{n}_0$.

During training, the model $\vfield_\theta$ is optimized through conditional flow matching to regress this vector field using a mean squared error (MSE) objective:
\begin{equation}
    \mathcal{L}_{\text{CFM}} = \mathbb{E}_{t, \vz^n_0\sim p_\text{data},\vz^{n}_t \sim p_t(\vz^{n}_t | \vz^{n}_0)} \|\vfield_\theta(\vz^{n}_t, t; \vz_0^{<n}) - \vfield^n_t \|^2_2,
\end{equation}
where $p_t(\vz^{n}_t | \vz^{n}_0)$ denotes the time-dependent probability path induced by the linear interpolation in Eq.~\eqref{eq:interpolation}.
Then, the minimizer of this objective corresponds to the conditional expectation:
\begin{equation}
    \vfield_\theta(\vz^n_t, t; \vz_0^{<n}) = \mathbb{E}[\vz^n_1 - \vz^n_0 | \vz^n_t, \vz_0^{<n}].
\end{equation}
At inference, given a noisy video chunk $\vz^{n}_t$ at timestep $t$, the model predicts the vector field $\vfield_\theta(\vz^{n}_t, t; \vz_0^{<n})$. By rearranging Eq.~\eqref{eq:interpolation}, we can obtain a clean estimate $\hat{\vz}^{n}_{0|t}$, which approximates the posterior mean $\mathbb{E}[\vz^{n}_0 \mid \vz^{n}_t, \vz_0^{<n}]$:
\begin{equation}
    \label{eq:denoised_pred}
    \hat{\vz}^{n}_{0|t} := \mathbb{E}[\vz^{n}_0 \mid \vz^{n}_t, \vz_0^{<n}]
    = \vz^{n}_t - t \vfield_\theta(\vz^{n}_t, t; \vz_0^{<n}). 
\end{equation}
In practice, to avoid the computational overhead of processing the full past context at each step, the conditioning on $\vz_0^{<n}$ is efficiently implemented using its KV cache~\cite{yin2025slow, huang2025self}, denoted as $\text{KV}^{<n}$.
Subsequently, following the specific sampling procedure of our backbone~\cite{huang2025self}, the process continues to the next timestep $t - \Delta t$ by re-noising the clean estimate $\hat{\vz}^n_{0|t}$ from Eq.~\eqref{eq:denoised_pred}:
\begin{equation}
    \label{eq:self_forcing_reverse_step}
    \vz^n_{t-\Delta t} = (1-(t-\Delta t)) \hat{\vz}^n_{0|t} + (t-\Delta t)\vz_1^n,
    \quad \vz_1^n \sim \mathcal{N}(\mathbf{0}, \mathbf{I}).
\end{equation}

\section{AVIS: Autoregressive Video Inverse Problem Solver}
\label{sec:avis}

In this section, we detail the architecture of our Autoregressive Video Inverse problem Solver (AVIS) and its highly efficient variant, AVIS \emph{Flash}, designed to harness the underexplored potential of AR video diffusion models for zero-shot restoration.
We begin by introducing an efficient mechanism to enforce measurement consistency during the reverse diffusion process in Section~\ref{sec:dds}.
Building upon this, Section~\ref{sec:avis_framework} presents our core framework, AVIS, which accelerates the process by initializing the reverse diffusion with a measurement-consistent estimate, thereby reducing the required sampling steps.
Finally, in Section~\ref{sec:avis_flash}, we propose AVIS \emph{Flash}, which pushes efficiency further by eliminating multiple VAE passes for subsequent video chunks.

\subsection{Guidance for Measurement Consistency}
\label{sec:dds}

During the reverse diffusion process of the $n$-th noisy video chunk $\vz^n_{t}$, measurement consistency is enforced at each denoising step to guide the sampling trajectory toward the posterior.
The conventional approach is to compute the gradient of the measurement consistency term~\cite{kwon2025video, chung2023diffusion} (\eg, $\nabla_{\vz} \|\vy - \Ac(\mathcal{D}(\vz))\|^2$).
However, this requires backpropagation through both the VAE decoder $\mathcal{D}$ and the model $\vfield_\theta$, which is computationally prohibitive for high-dimensional video data. 

To ensure both efficiency and generalizability, we follow the common practice of recent DVIS~\cite{kwon2025solving, kwon2025vision, spagnoletti2026lvtino} by adopting a multi-step optimization algorithm for the update.
Specifically, we employ the Decomposed Diffusion Sampling (DDS)~\cite{chung2024decomposed} framework, which utilizes the conjugate gradient (CG) method~\cite{hestenes1952methods, liesen2013krylov} to efficiently enforce measurement consistency.

Formally, at each timestep $t$, we first obtain the clean latent estimate $\hat{\vz}^n_{0|t}$ from $\vz^n_t$ using Eq.~\eqref{eq:denoised_pred} and decode it into the pixel space:
\begin{equation}
    \hat{\vx}^n_{0|t} = \mathcal{D}(\hat{\vz}^n_{0|t}).
\end{equation}
Next, we enforce data consistency by solving the following proximal optimization problem through the aforementioned CG method:
\begin{equation}
    \label{eq:proximal}
    \tilde{\vx}^n_{0|t} := \arg\min_{\vx} \frac{\gamma}{2} \|\vy - \Ac(\vx)\|^2 + \frac{1}{2} \|\vx - \hat{\vx}^n_{0|t}\|^2.
\end{equation}
To avoid complex parameter tuning, we fix $\gamma = 1$ throughout this paper. 
We then re-encode the updated pixel-space estimate $\tilde{\vx}^n_{0|t}$ back into the latent space:
\begin{equation}
    \tilde{\vz}^n_{0|t} = \mathcal{E}(\tilde{\vx}^n_{0|t}).
\end{equation}
By doing so, we enforce measurement consistency using only forward passes of $\mathcal{D}$ and $\mathcal{E}$, entirely avoiding the expensive neural network Jacobian computation. 
Finally, the updated latent estimate $\tilde{\vz}^n_{0|t}$ is renoised back to the next timestep $t - \Delta t$ using Eq.~\eqref{eq:self_forcing_reverse_step}, 
\begin{equation}
    \label{eq:renoise}
    \vz^n_{t-\Delta t} = (1 - (t-\Delta t))\tilde{\vz}^n_{0|t} + (t-\Delta t)\vz^n_1, \quad \vz^n_1\sim \mathcal{N}(\mathbf{0},\mathbf{I}),
\end{equation}
to continue the reverse process.

\subsection{AVIS: Accelerating DVIS with a Better Starting Point}
\label{sec:avis_framework}
Diffusion models are inherently slow, as they require numerous steps to generate data from pure Gaussian noise.
To accelerate solving inverse problems, CCDF~\cite{chung2022come} proposes initializing the reverse diffusion process from a \emph{better} starting point.
Specifically, it obtains a coarse estimate from an auxiliary pre-trained network, diffuses it to an intermediate timestep $t_0 < 1$, and starts the reverse sampling from there.
By starting at $t_0$, CCDF effectively reduces the total number of sampling steps, thereby significantly reducing inference time while leveraging the stochastic contraction property of the reverse diffusion path.

However, its reliance on a task-specific auxiliary network for the initial correction fundamentally sacrifices the zero-shot generality of the solver.
To overcome this limitation, we propose a training-free, multi-step optimization strategy to obtain a better starting point, preserving the acceleration benefits of the CCDF principle.

Specifically, we obtain the starting point of the reverse sampling as follows.
We first compute a measurement-consistent initial estimate $\vxinit^{\lambda}$ in pixel space using a multi-step optimization algorithm (\eg, CG) that minimizes the measurement consistency objective:
\begin{equation}
    \label{eq:cg}
    \vxinit^{\lambda} = \arg\min_{\vx} \frac{1}{2}\|\vy - \Ac(\vx)\|_2^2+
    \frac{\lambda}{2}\left\|\vx-\vx_{\text{ref}}\right\|_2^2,
\end{equation}
where $\vx_{\text{ref}}$ denotes the dimension-aligned measurement, \ie, $\Ac^{\top}(\vy)$.
Next, we encode $\vxinit^{\lambda}$ to get $\vzinit$ and diffuse it to timestep $t_0$:
\begin{equation}
    \label{eq:init_noisy}
    \vz_{t_0} = [\vz^1_{t_0}, \vz^2_{t_0}, \dots, \vz^N_{t_0}] = (1-t_0)\vzinit + t_0\vz_1, \quad \vz_1 \sim \mathcal{N}(\mathbf{0}, \mathbf{I}).
\end{equation}

While this initial estimate $\vxinit^{\lambda}$ satisfies measurement consistency, it is only a coarse estimate and may lack perceptual realism.
The role of the reverse diffusion process is to transport this estimate onto the high-density region of the video data manifold.
To this end, we perform $K$ sampling steps starting from $t_0$, yielding a step size of $\Delta t = t_0 / K$ for each reverse diffusion process.
During this process, we continuously apply DDS guidance at each step to ensure the trajectory remains strictly faithful to the given measurement.
By integrating the guidance and the accelerated sampling strategy into the autoregressive diffusion process, we establish our complete framework: the \textbf{A}utoregressive \textbf{V}ideo \textbf{I}nverse problem \textbf{S}olver (\textbf{AVIS}).

\subsection{AVIS \emph{Flash}: Accelerating AVIS with Autoregressive Propagation}
\label{sec:avis_flash}
Here, we propose AVIS \emph{Flash}, a highly accelerated variant that fully unlocks the potential of autoregressive propagation. 
While it shares the same initialization at $t_0$ with AVIS, it fundamentally eliminates the need for iterative measurement guidance after the first chunk.
Although the guidance introduced in Section~\ref{sec:dds} avoids expensive neural network Jacobian computations, applying it to every video chunk requires iterative VAE passes at each reverse step, which remains a bottleneck for high-throughput video restoration.

Empirically, we discover a strikingly simple yet effective solution: bypassing measurement guidance for all subsequent chunks ($n \ge 2$) dramatically accelerates throughput with competitive restoration performance.
We attribute this success to the temporal context propagation induced by the autoregressive formulation, combined with the initialization.
Once the first video chunk ($n=1$) is grounded to the measurements through guidance, its restored features are extracted and stored as a prefix cache $\KV^1$.
Furthermore, since the initialization in Section~\ref{sec:avis_framework} provides a measurement-consistent estimate that is then re-noised to an intermediate timestep $t_0$, the model does not need to generate subsequent chunks ($n \ge 2$) from scratch; it simply refines this estimate through the guidance-free sampling steps in Eq.~\eqref{eq:denoised_pred} and Eq.~\eqref{eq:self_forcing_reverse_step}, conditioned on the accumulated prefix cache $\KV^{<n}$.

To clarify the intuition behind AVIS \emph{Flash}, we provide an error decomposition within our framework.
Specifically, Proposition~\ref{obs:chunkwise_stability} decomposes the final error into two terms: the initial error from the measurement-consistent estimate, and the context error propagated from previously restored chunks.
Detailed derivations are provided in Appendix~\ref{appendix:proof}.

\begin{restatable}{proposition}{bound}
    \label{obs:chunkwise_stability}
    Let $\epsilon_0^n$ denote the initial error of the current chunk, and let $\delta_n$ denote the context error propagated from previously restored chunks.
    Under local Lipschitz assumptions on the autoregressive vector field $\vfield_\theta$, the final error $\epsilon_K^n$ of the $n$-th chunk after $K$ sampling steps satisfies
    \begin{equation}
        \epsilon_K^n \leq \Lambda_K \epsilon_0^n + B_K \delta_n,
    \end{equation}
    where $\Lambda_K$ and $B_K$ depend on the sampling schedule and local Lipschitz constants.
\end{restatable}

This decomposition provides an interpretive view for understanding the design of AVIS \emph{Flash}.
The measurement-consistent initialization is designed to reduce the per-chunk initial error $\epsilon_0^n$ before autoregressive sampling.
Meanwhile, the context error $\delta_n$ contributes to the bound as an additive perturbation term.
This indicates that the context error propagated from previous chunks affects the current chunk through incremental additions rather than through multiplicative amplification of the initial error.

AVIS \emph{Flash} can also be naturally extended to handle long video sequences.
By periodically re-injecting measurement consistency (\eg, applying guidance every $N$ chunks), we can further suppress error accumulation.
This periodic guidance effectively mitigates drift, allowing for more stable restoration over extended video sequences, as demonstrated in Appendix~\ref{appendix:more_experiments}.
The complete AVIS \emph{Flash} pipeline is detailed in Algorithm~\ref{alg:avis_flash}.

\begin{algorithm}[t]
    \caption{Autoregressive Video Inverse problem Solver (\textcolor{custompinktext}{AVIS} and \textcolor{custombluetext}{AVIS \emph{Flash}})}
    \label{alg:avis_flash}
    \small
    \begin{algorithmic}[1]
        \Require Measurement $\vy$, operator $\Ac$, diffusion model $\vfield_\theta$, number of chunks $N$, start time $t_0$, encoder $\mathcal{E}$, decoder $\mathcal{D}$, and mode $\in \{\text{\textcolor{custompinktext}{AVIS}}, \text{\textcolor{custombluetext}{AVIS \emph{Flash}}}\}$.
        \State $\vxinit^{\lambda} \leftarrow \text{Solve Eq.~\eqref{eq:cg}}$ \textbf{and} $\vzinit \leftarrow \mathcal{E}(\vxinit^{\lambda})$ \Comment{Pre-restoration}
        \State $\KV^0 \leftarrow \emptyset$
        \For{$n = 1:N$}
          \State $\vz_{t_0}^n \leftarrow (1-t_0)\vzinit^{n} + t_0 \vz_1, \quad \vz_1 \sim \mathcal{N}(\mathbf{0}, \mathbf{I})$ \Comment{Initialization}
          \For{$t : t_0 \rightarrow 0$}
            \State $\hat\vz_{0|t}^n \leftarrow \vz_t^n - t\vfield_\theta(\vz_t^n, t;\, \KV^{<n})$ \Comment{Obtain clean estimate}
            
            \If{mode \textbf{is} \textcolor{custompinktext}{AVIS} \textbf{or} $(\text{mode \textbf{is} \textcolor{custombluetext}{AVIS \emph{Flash}} and } n = 1)$}
                \State $\hat{\vz}^n_{0|t} \gets \text{Solve Eq.~\eqref{eq:proximal} via CG}$ \Comment{Measurement consistency update}
            \EndIf
            \State $\vz^n_{t-\Delta t} \leftarrow (1 - (t-\Delta t))\hat\vz_{0|t}^n + (t-\Delta t)\vz_1, \quad \vz_1 \sim \mathcal{N}(\mathbf{0}, \mathbf{I})$ \Comment{Re-noise}
        
          \EndFor
          \State \textbf{Display} $\mathcal{D}(\vz_{0}^n)$ \textbf{and} $\KV^{\leq n} \leftarrow \text{Update KV}(\vz_{0}^n,\, \KV^{<n})$
        \EndFor
        \State \Return $[\vz_{0}^1, \dots, \vz_{0}^N]$
    \end{algorithmic}
\end{algorithm}

\section{Experiments}
\label{sec:experiments}

\subsection{Experimental Setup}
\noindent\textbf{Dataset \& Inverse Problems.}
We evaluate our method on 100 high-resolution videos sourced from Pexels~\cite{pexels2024}.
Each video is spatially resized to $480\times832$ and temporally cropped to the initial 81 frames.
We comprehensively evaluate our method on five video restoration tasks: (1) \emph{Super-Resolution} at a 4$\times$ scaling factor, (2) \emph{Random Inpainting} with 50\% of the video masked, (3) \emph{Gaussian Deblur} with a kernel size of 61 and a $\sigma$ of 3.0, (4) \emph{Temporal Average} using 7$\times$ averaging over a 7-frame causal window, and (5) \emph{Spatio-Temporal Average} using 4$\times$ spatial average pooling and 4$\times$ temporal averaging over a 4-frame causal window.

\noindent\textbf{Baselines.}
We compare our methods (AVIS and AVIS \emph{Flash}) against other diffusion model-based solvers for high-resolution video inverse problems: DiffIR2VR-Zero~\cite{yeh2024diffir2vr}, VISION-XL~\cite{kwon2025vision}, and LVTINO~\cite{spagnoletti2026lvtino}.
We include DiffIR2VR-Zero specifically for super-resolution (SR), as it only supports SR among the various inverse problems addressed in this work.
For completeness, detailed descriptions of each baseline method are provided in Appendix~\ref{appendix:details}.

\noindent\textbf{Metrics \& Implementation Details.}
To assess restoration quality, we report five fidelity and perceptual metrics: PSNR, SSIM~\cite{wang2004image}, LPIPS~\cite{zhang2018unreasonable}, FID~\cite{heusel2017gans}, and FVD~\cite{unterthiner2019fvd}.
We further evaluate video quality using VBench~\cite{huang2024vbench}, specifically focusing on Subject Consistency (Sub. Con.), Background Consistency (Bg. Con.), Motion Smoothness (M. Smooth.), Aesthetic Quality (Aesth.), and Imaging Quality (Imag.).
Computational efficiency is measured in terms of latency (time to restore the first frame in seconds), total restoration time, and throughput (FPS).
Unless otherwise specified, all experiments are conducted on a single NVIDIA RTX 4090 GPU.
Based on our ablation studies, we employ $t_0=0.1$ and $K=2$ as our default settings.
Please refer to Appendix~\ref{appendix:details} for more information.

\subsection{Results}

\begin{table}[b]
\centering
\small
\resizebox{0.5\columnwidth}{!}{
\begin{tabular}{l ccc}
\toprule
Method & Latency (s) $\downarrow$ & Time (s) $\downarrow$ & FPS (frame/s)  $\uparrow$ \\
\midrule
DiffIR2VR-Zero & 1300 & 1300 & 0.06 \\
VISION-XL    & 167 & 167 & 0.49 \\
LVTINO       & \underline{114} & 114 & 0.71 \\
\rowcolor{custompink!20}
AVIS         & \textbf{4} & \underline{68.5} & \underline{1.18} \\
\rowcolor{customblue!30}
AVIS \emph{Flash}   & \textbf{4} & \textbf{13.7} & \textbf{5.91} \\
\bottomrule
\end{tabular}
}
\vspace{0.5em}
\caption{%
\textbf{Efficiency comparison for $4\times$ video super-resolution.} Our framework (\textcolor{custompinktext}{AVIS} and \textcolor{custombluetext}{AVIS \emph{Flash}}) achieves significant improvements across all efficiency metrics. \textbf{Bold} and \underline{underline} indicate the best and second-best results, respectively.
}
\label{tab:efficiency}
\vspace{-1.5em}
\end{table}

\noindent\textbf{Computational Efficiency.}
The primary advantage of our framework is its practical efficiency.
As shown in Table~\ref{tab:efficiency}, both AVIS and AVIS \emph{Flash} reduce the initial latency to only 4 seconds, whereas the baselines suffer from severe delays (167 seconds for VISION-XL and 114 seconds for LVTINO).
Beyond latency improvements, our framework significantly enhances overall throughput.
While the baselines fail to surpass 1 FPS, AVIS successfully achieves over 1 FPS.
Notably, AVIS \emph{Flash} achieves 5.91 FPS, making it over $12\times$ and $8\times$ faster than VISION-XL and LVTINO, respectively.
Furthermore, we found that deploying AVIS \emph{Flash} on a single NVIDIA H100 GPU pushes the performance even further to an initial latency of 1.85 seconds and a throughput of 10.2 FPS, paving the way toward real-time deployment.

\begin{table*}[t]
\centering
\small
\resizebox{\textwidth}{!}{%
\begin{tabular}{l | ccccc | ccccc}
\toprule
Method & PSNR$\uparrow$ & SSIM$\uparrow$ & LPIPS$\downarrow$ & FVD$\downarrow$ & FID$\downarrow$ & Sub. Con.$\uparrow$ & Bg. Con.$\uparrow$ & M. Smooth.$\uparrow$ & Aesth.$\uparrow$ & Imag.$\uparrow$ \\
\midrule
\multicolumn{11}{l}{\textbf{\emph{Super Resolution}}} \\
\quad DiffIR2VR-Zero & 27.46 & 0.722 & 0.157 & 225.6 & 24.49 & 95.92 & 95.22 & 98.14 & \textbf{56.00} & \textbf{69.09} \\
\quad VISION-XL     & 29.89 & 0.800 & 0.115 & 76.24 & \underline{22.51} & 96.10 & 95.75 & 99.05 & 51.96 & 53.73 \\
\quad LVTINO        & \underline{30.04} & \underline{0.824} & \underline{0.102} & \underline{59.42} & \textbf{18.32} & \textbf{96.36} & 96.09 & 99.11 & \underline{52.63} & \underline{59.05} \\
\rowcolor{custompink!20}
\quad AVIS          & \textbf{30.38} & \textbf{0.826} & \textbf{0.101} & \textbf{40.36} & 23.94 & \underline{96.30} & \textbf{96.21} & \textbf{99.24} & 52.19 & 55.13 \\
\rowcolor{customblue!30}
\quad AVIS \emph{Flash}    & 29.95 & 0.818 & 0.109 & 62.98 & 25.46 & 96.29 & \underline{96.14} & \underline{99.15} & 51.80 & 54.59 \\
\midrule
\multicolumn{11}{l}{\textbf{\emph{Inpainting}}} \\
\quad VISION-XL     & 29.39 & 0.801 & 0.123 & 144.8 & 18.92 & 96.09 & 94.51 & 98.95 & 52.82 & 60.37 \\
\quad LVTINO        & 26.36 & 0.737 & 0.187 & 322.0 & 35.75 & 95.95 & \textbf{95.56} & 98.60 & 51.94 & \textbf{66.86} \\
\rowcolor{custompink!20}
\quad AVIS          & \textbf{31.27} & \textbf{0.870} & \textbf{0.075} & \textbf{68.90} & \textbf{10.21} & \textbf{96.24} & 94.61 & \textbf{99.20} & \textbf{54.06} & \underline{64.15} \\
\rowcolor{customblue!30}
\quad AVIS \emph{Flash}    & \underline{30.29} & \underline{0.842} & \underline{0.091} & \underline{81.75} & \underline{13.96} & \underline{96.18} & \underline{94.84} & \underline{99.14} & \underline{53.22} & 62.90 \\
\midrule
\multicolumn{11}{l}{\textbf{\emph{Gaussian Deblur}}} \\
\quad VISION-XL     & \textbf{31.10} & \textbf{0.835} & \textbf{0.093} & \underline{41.52} & \textbf{16.38} & 96.27 & 96.23 & 99.08 & \textbf{52.52} & \textbf{58.01} \\
\quad LVTINO        & 30.64 & \textbf{0.835} & \underline{0.109} & 47.45 & \underline{20.37} & \textbf{96.43} & 96.18 & 99.20 & \underline{51.65} & \underline{57.17} \\
\rowcolor{custompink!20}
\quad AVIS          & \underline{30.69} & 0.831 & 0.110 & \textbf{36.07} & 24.40 & 96.30 & \textbf{96.41} & \textbf{99.25} & 51.54 & 54.63 \\
\rowcolor{customblue!30}
\quad AVIS \emph{Flash}    & 30.41 & 0.827 & 0.115 & 49.66 & 25.29 & \underline{96.31} & \underline{96.33} & \underline{99.23} & 51.24 & 54.39 \\
\midrule
\multicolumn{11}{l}{\textbf{\emph{Temporal Average}}} \\
\quad VISION-XL     & 30.46 & 0.825 & 0.095 & 110.5 & 12.15 & 95.87 & 95.00 & 98.67 & 54.27 & 63.08 \\
\quad LVTINO        & \underline{31.68} & \underline{0.878} & 0.069 & 77.76 & 8.015 & \textbf{96.34} & 95.11 & 99.01 & \textbf{55.73} & \underline{66.36} \\
\rowcolor{custompink!20}
\quad AVIS          & \textbf{32.10} & \textbf{0.887} & \textbf{0.063} & \textbf{54.85} & \textbf{6.803} & \underline{96.26} & \underline{95.58} & \textbf{99.12} & \underline{55.29} & 67.07 \\
\rowcolor{customblue!30}
\quad AVIS \emph{Flash}    & 31.57 & \underline{0.878} & \underline{0.067} & \underline{57.92} & \underline{6.829} & 96.21 & \textbf{95.65} & \underline{99.05} & 55.16 & \textbf{67.54} \\
\midrule
\multicolumn{11}{l}{\textbf{\emph{Spatio-Temporal Average}}} \\
\quad VISION-XL     & 29.62 & 0.794 & 0.124 & 99.54 & \underline{24.13} & 95.96 & 95.41 & 98.96 & 51.44 & 52.73 \\
\quad LVTINO        & \underline{29.85} & \textbf{0.818} & \textbf{0.110} & 95.35 & \textbf{19.73} & \textbf{96.27} & \underline{95.98} & 99.03 & \textbf{52.42} & \textbf{58.02} \\
\rowcolor{custompink!20}
\quad AVIS          & \textbf{29.99} & \textbf{0.818} & \underline{0.112} & \textbf{62.11} & 25.61 & \underline{96.22} & \textbf{96.05} & \textbf{99.20} & \underline{51.62} & \underline{54.28} \\
\rowcolor{customblue!30}
\quad AVIS \emph{Flash}    & 29.81 & \underline{0.815} & 0.117 & \underline{73.23} & 25.94 & \underline{96.22} & 95.97 & \underline{99.13} & 51.32 & 54.20 \\
\bottomrule
\end{tabular}
}
\vspace{-0.5em}
\caption{%
    \textbf{Quantitative comparison across five video restoration tasks.} Our proposed \textcolor{custompinktext}{AVIS} achieves superior performance, and its highly efficient variant, \textcolor{custombluetext}{AVIS \emph{Flash}}, maintains competitive performance. \textbf{Bold} and \underline{underline} indicate the best and second-best results, respectively.
}
\vspace{-0.3em}
\label{tab:main}
\end{table*}

\noindent\textbf{Restoration Performance.}
Table~\ref{tab:main} compares AVIS and AVIS \emph{Flash} against baselines across five restoration tasks.
Overall, AVIS achieves the best restoration performance, outperforming or remaining highly competitive with all baselines on the majority of metrics.
In particular, AVIS consistently attains the best or near-best PSNR, SSIM, LPIPS, and FVD across most tasks, validating the effectiveness of combining autoregressive video diffusion priors with our initialization.
While LVTINO occasionally achieves favorable perceptual metrics, it relies on both a video diffusion prior~\cite{yin2025slow} and an additional image diffusion prior~\cite{yin2024improved}.
In contrast, AVIS achieves better overall results with only a single backbone.

AVIS \emph{Flash} targets a faster operating point.
By bypassing iterative measurement updates for subsequent chunks, AVIS \emph{Flash} experiences a slight drop in certain metrics compared to AVIS.
However, it remains competitive with baselines such as VISION-XL and LVTINO, while reducing initial latency by over 100 seconds and boosting throughput by over $12\times$ and $8\times$, respectively (Table~\ref{tab:efficiency}).
These results highlight the central benefit of our framework: AVIS establishes a high-quality autoregressive video inverse problem solver, while AVIS \emph{Flash} offers a favorable efficiency--performance trade-off for practical deployment.

\noindent\textbf{Robustness to Measurement Noise.}
Beyond the noise-free setting ($\sigma_n{=}0$, $\lambda{=}0$) in Table~\ref{tab:main}, we further evaluate AVIS and AVIS \emph{Flash} under additive Gaussian measurement noise ($\sigma_n{=}0.01$) on all five tasks in Table~\ref{tab:noise}.
We fix $\lambda{=}0.1$ in Eq.~\eqref{eq:cg} across all tasks without per-task tuning.
AVIS achieves the best PSNR on every task and the best results on the majority of the remaining metrics, while AVIS \emph{Flash} remains competitive with LVTINO.
These results demonstrate that our framework remains robust to measurement noise.

\noindent\textbf{Further Experiments.}
In Appendix~\ref{appendix:more_experiments}, we provide further results, including long video restoration (\S\ref{appendix:long_video}), the robustness of AVIS \emph{Flash} on long videos (\S\ref{appendix:robustness}), novel view synthesis formulated as another type of inpainting problem (\S\ref{appendix:nvs}), spatio-temporal super-resolution on GoPro240 and Adobe240 (\S\ref{appendix:stsr}), chunk-wise measurement consistency (\S\ref{appendix:residual}), a matched-backbone comparison with LVTINO (\S\ref{appendix:matched_backbone}), and the detailed runtime for each degradation (\S\ref{appendix:runtime}).

\begin{table}[t]
    \centering
    \small
    \setlength{\abovecaptionskip}{4pt}
    \resizebox{0.5\linewidth}{!}{%
    \begin{tabular}{l | ccccc}
        \toprule
        Method & PSNR$\uparrow$ & SSIM$\uparrow$ & LPIPS$\downarrow$ & FID$\downarrow$ & FVD$\downarrow$ \\
        \midrule
        \multicolumn{6}{l}{\textbf{\emph{Super Resolution}}} \\
        \quad LVTINO        & \underline{29.77} & \underline{0.813} & \underline{0.125} & \textbf{23.63} & \underline{67.00} \\
        \rowcolor{custompink!20}
        \quad AVIS          & \textbf{30.02} & \textbf{0.817} & \textbf{0.123} & \underline{28.20} & \textbf{57.70} \\
        \rowcolor{customblue!30}
        \quad AVIS \emph{Flash} & 29.57 & 0.808 & 0.133 & 29.13 & 83.81 \\
        \midrule
        \multicolumn{6}{l}{\textbf{\emph{Inpainting}}} \\
        \quad LVTINO        & 26.39 & 0.741 & 0.188 & 34.36 & 382.3 \\
        \rowcolor{custompink!20}
        \quad AVIS          & \textbf{31.09} & \textbf{0.869} & \textbf{0.076} & \textbf{9.850} & \textbf{65.70} \\
        \rowcolor{customblue!30}
        \quad AVIS \emph{Flash} & \underline{29.31} & \underline{0.819} & \underline{0.099} & \underline{12.03} & \underline{97.16} \\
        \midrule
        \multicolumn{6}{l}{\textbf{\emph{Gaussian Deblur}}} \\
        \quad LVTINO        & \underline{26.19} & \textbf{0.694} & \textbf{0.267} & \underline{55.89} & 520.7 \\
        \rowcolor{custompink!20}
        \quad AVIS          & \textbf{26.80} & \underline{0.602} & \underline{0.287} & \textbf{55.75} & \textbf{215.3} \\
        \rowcolor{customblue!30}
        \quad AVIS \emph{Flash} & 25.46 & 0.563 & 0.317 & 61.76 & \underline{321.7} \\
        \midrule
        \multicolumn{6}{l}{\textbf{\emph{Temporal Average}}} \\
        \quad LVTINO        & 29.04 & 0.788 & 0.144 & \textbf{18.42} & \textbf{121.6} \\
        \rowcolor{custompink!20}
        \quad AVIS          & \textbf{29.90} & \textbf{0.843} & \textbf{0.124} & \underline{22.48} & \underline{182.1} \\
        \rowcolor{customblue!30}
        \quad AVIS \emph{Flash} & \underline{29.15} & \underline{0.826} & \underline{0.136} & 23.96 & 194.5 \\
        \midrule
        \multicolumn{6}{l}{\textbf{\emph{Spatio-Temporal Average}}} \\
        \quad LVTINO        & 27.10 & 0.701 & 0.255 & 51.76 & 745.7 \\
        \rowcolor{custompink!20}
        \quad AVIS          & \textbf{28.84} & \textbf{0.789} & \textbf{0.163} & \textbf{36.40} & \textbf{211.3} \\
        \rowcolor{customblue!30}
        \quad AVIS \emph{Flash} & \underline{28.65} & \underline{0.786} & \underline{0.167} & \underline{37.17} & \underline{250.9} \\
        \bottomrule
    \end{tabular}
    }
    \caption{%
        \textbf{Robustness to measurement noise} ($\sigma_n{=}0.01$, $\lambda{=}0.1$). \textcolor{custompinktext}{AVIS} and \textcolor{custombluetext}{AVIS \emph{Flash}} remain stable under measurement noise. \textbf{Bold} and \underline{underline} indicate the best and second-best results, respectively.
    }
    \vspace{-0.5em}
    \label{tab:noise}
\end{table}

\subsection{Ablation Study}
\label{sec:ablation}
To isolate the contributions of our core components, we report the ablation study results for video inpainting on the initialization for AVIS in Table~\ref{tab:init}, and on AVIS \emph{Flash} in Table~\ref{tab:ablation}.
For ablation results on the other tasks, please refer to Appendix~\ref{appendix:more_ablation}.

\noindent\textbf{Initialization.}
To separate the contribution of the solver from the quality of the initialization, we compare AVIS with the initial estimate $\vxinit^{\lambda}$, its VAE reconstruction $\mathcal{D}(\mathcal{E}(\vxinit^{\lambda}))$, and AVIS without DDS guidance.
We conduct this experiment under the strong noise setting ($\sigma_n{=}0.05$, $\lambda{=}0.3$).
As shown in Table~\ref{tab:init}, AVIS achieves the best results on all metrics.
The VAE reconstruction yields only marginal changes over $\vxinit^{\lambda}$, indicating that the gain does not come from the encoder $\mathcal{E}$ and decoder $\mathcal{D}$ alone.
Without DDS guidance, the fidelity metrics fall even below those of $\vxinit^{\lambda}$, as the generative prior alone drifts away from the measurements.
These results demonstrate that the final restoration cannot be explained by the initialization alone, and that the generative prior and measurement guidance synergistically refine the initial estimate.

\begin{table}[t]
    \centering
    \small
    \setlength{\abovecaptionskip}{4pt}
    \resizebox{0.5\linewidth}{!}{%
    \begin{tabular}{l | ccccc}
        \toprule
        Setting & PSNR$\uparrow$ & SSIM$\uparrow$ & LPIPS$\downarrow$ & FID$\downarrow$ & FVD$\downarrow$ \\
        \midrule
        $\vxinit^{\lambda}$                          & \underline{30.32} & \underline{0.826} & 0.145 & 26.57 & 136.8 \\
        $\mathcal{D}(\mathcal{E}(\vxinit^{\lambda}))$ & 30.11 & 0.818 & \underline{0.136} & \underline{24.31} & 130.3 \\
        AVIS w/o DDS                                  & 29.04 & 0.789 & 0.150 & 25.42 & \underline{129.5} \\
        \rowcolor{gray!25}
        AVIS                                          & \textbf{30.35} & \textbf{0.834} & \textbf{0.121} & \textbf{18.52} & \textbf{100.8} \\
        \bottomrule
    \end{tabular}
    }
    \vspace{0.2em}
    \caption{%
        \textbf{Effect of the initialization on video inpainting} ($\sigma_n{=}0.05$, $\lambda{=}0.3$). AVIS consistently improves over the initialization and its VAE reconstruction.
    }
    \vspace{-1.0em}
    \label{tab:init}
\end{table}

\noindent\textbf{Autoregressive Propagation.}
To isolate the contribution of autoregressive context propagation, we remove the KV cache for subsequent chunks while keeping the rest of AVIS \emph{Flash} unchanged.
Under this setting, each subsequent chunk starts from the same initialized noisy state as in AVIS \emph{Flash}, but is restored without access to the previously generated prefix.
As shown in Table~\ref{tab:ablation}(a), removing the KV cache consistently degrades all metrics.
This highlights the benefit of autoregressive propagation, which provides performance gains complementary to the initialization.
In other words, the autoregressive propagation of the restored prefix further improves subsequent chunk restoration and maintains temporal coherence across the video.

To further isolate the role of autoregressive propagation, we visualize the case where subsequent chunks are generated with KV cache conditioning but without the initialization (\ie, starting from pure noise).
As shown in the middle row of Figure~\ref{fig:ar_prop}, autoregressive propagation alone effectively preserves the context from previous chunks, but gradually drifts from the desired restoration.
This drift is mitigated by our initialization, as shown in the bottom row of Figure~\ref{fig:ar_prop}.
Starting the reverse process from the initialized anchor keeps each chunk closer to the desired restoration trajectory, while periodic re-injection of measurement consistency further prevents long-range drift for longer videos, as described in Section~\ref{sec:avis_flash}.

\noindent\textbf{Start Time ($t_0$).}
We study the effect of the start time $t_0 \in \{0.1, 0.2, 0.5\}$.
As shown in Table~\ref{tab:ablation}(b), a smaller $t_0$ consistently improves the majority of metrics, with $t_0=0.1$ yielding the best overall performance.
This is consistent with the results in Figure~\ref{fig:ar_prop}: as $t_0$ increases, the starting point approaches pure noise, and the reverse process increasingly relies on pure autoregressive propagation, which gradually drifts away from the desired restoration.
Although a larger $t_0$ slightly improves some perceptual VBench metrics, it substantially degrades fidelity metrics.
We therefore adopt $t_0=0.1$ as our default setting.

\noindent\textbf{Sampling Steps ($K$).}
We also evaluate the impact of the number of sampling steps ($K \in \{1, 2, 4\}$).
Table~\ref{tab:ablation}(c) reveals a clear trade-off: as the number of steps increases, certain perceptual metrics (\eg, FID) improve, whereas pixel-level fidelity metrics (PSNR, SSIM) slightly decrease.
Since larger $K$ increases inference time with diminishing returns in overall quality, we adopt $K=2$ as the default setting to achieve the best balance between speed and quality.

\begin{table*}[t]
\centering
\small
\resizebox{\textwidth}{!}{%
\begin{tabular}{l | ccccc | ccccc}
\toprule
Setting & PSNR$\uparrow$ & SSIM$\uparrow$ & LPIPS$\downarrow$ & FVD$\downarrow$ & FID$\downarrow$ & Sub. Con.$\uparrow$ & Bg. Con.$\uparrow$ & M. Smooth.$\uparrow$ & Aesth.$\uparrow$ & Imag.$\uparrow$ \\
\midrule
\multicolumn{11}{l}{(a) \textbf{\emph{Autoregressive Propagation}}} \\[0.2em]
\rowcolor{gray!25}
\quad w/ AR Prop (default) & \textbf{30.29} & \textbf{0.842} & \textbf{0.091} & \textbf{81.75} & \textbf{13.96} & \textbf{96.18} & \textbf{94.84} & \textbf{99.14} & \textbf{53.22} & \textbf{62.90} \\
\quad w/o AR Prop         & 29.75 & 0.830 & 0.103 & 108.3 & 17.67 & 95.99 & 94.82 & 99.11 & 52.97 & 62.36 \\
\midrule
\multicolumn{11}{l}{(b) \textbf{\emph{Start Time ($t_0$)}}} \\[0.2em]
\rowcolor{gray!25}
\quad $t_0=0.1$ (default)& \textbf{30.29} & \textbf{0.842} & \textbf{0.091} & \textbf{81.75} & \textbf{13.96} & \textbf{96.18} & \textbf{94.84} & \underline{99.14} & 53.22 & 62.90 \\
\quad $t_0=0.2$                 & \underline{28.98} & \underline{0.810} & \underline{0.112} & \underline{112.0} & \underline{16.48} & \underline{96.13} & \underline{94.53} & \underline{99.14} & \underline{53.28} & \underline{63.45} \\
\quad $t_0=0.5$                 & 25.50 & 0.716 & 0.192 & 283.3 & 30.50 & 95.92 & 94.11 & \textbf{99.15} & \textbf{53.31} & \textbf{65.21} \\
\midrule
\multicolumn{11}{l}{(c) \textbf{\emph{Sampling Steps ($K$)}}} \\[0.2em]
\quad $K=1$                     & \textbf{30.40} & \textbf{0.843} & \textbf{0.091} & \textbf{81.73} & 14.64 & \underline{96.18} & 94.73 & 99.13 & 53.15 & 62.47 \\
\rowcolor{gray!25}
\quad $K=2$ (default)  & \underline{30.29} & \underline{0.842} & \textbf{0.091} & \underline{81.75} & \underline{13.96} & \underline{96.18} & \underline{94.84} & \underline{99.14} & \underline{53.22} & \underline{62.90} \\
\quad $K=4$                     & 29.95 & 0.837 & \underline{0.094} & 81.78 & \textbf{13.69} & \textbf{96.23} & \textbf{95.17} & \textbf{99.15} & \textbf{53.34} & \textbf{63.46} \\
\bottomrule
\end{tabular}
}
\vspace{-0.5em}
\caption{%
  \textbf{Ablation studies on AVIS \emph{Flash} for the video inpainting task.} (a) Effect of autoregressive propagation with KV cache. (b) Impact of the start time $t_0$. (c) Impact of the number of diffusion sampling steps $K$. \textbf{Bold} and \underline{underline} indicate the best and second-best results, respectively.
}
\label{tab:ablation}
\end{table*}

\begin{figure}[t]
    \centering
    \includegraphics[width=\linewidth]{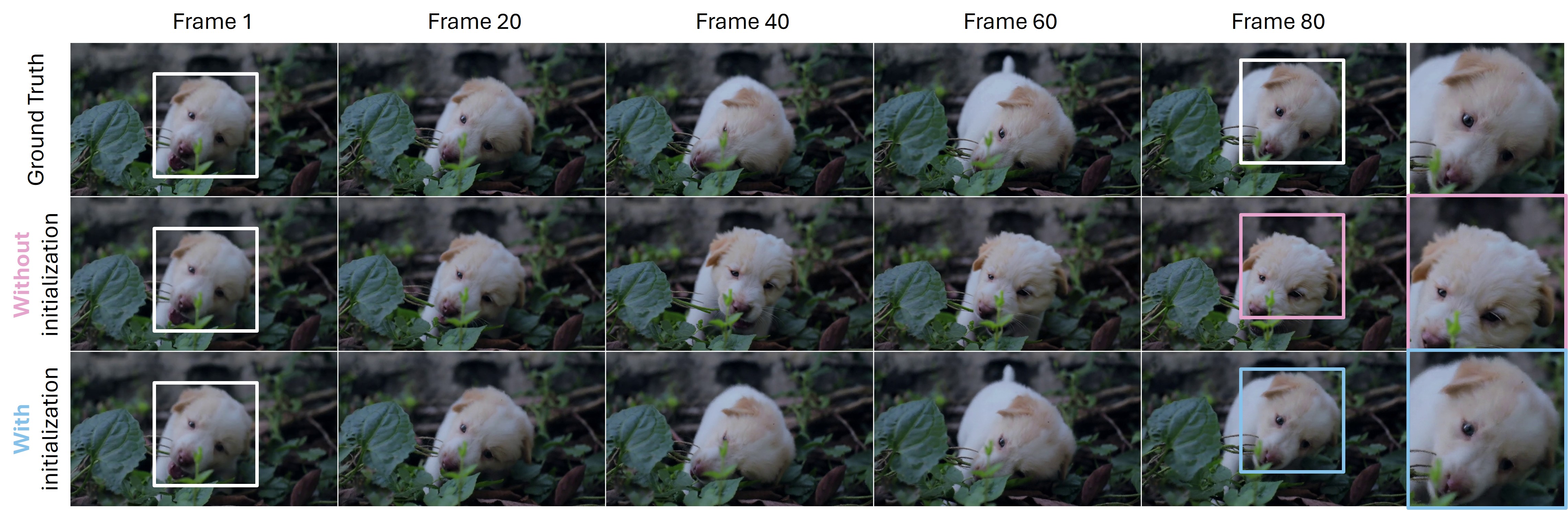}
    \vspace{-1.5em}
    \caption{\textbf{Effect of AR propagation and initialization.} (Top) Ground-truth video. (Middle) While AR propagation preserves the preceding context, it exhibits gradual error accumulation over time. (Bottom) Our initialization for the reverse diffusion effectively mitigates this temporal drift.}
    \label{fig:ar_prop}
    \vspace{-1em}
\end{figure}

\section{Conclusion}
\label{sec:conclusion}
In this paper, we proposed Autoregressive Video Inverse problem Solver (AVIS) and its highly efficient variant, AVIS \emph{Flash}, for practical deployment of diffusion model-based video inverse problem solvers.
{Extensive experiments show that the proposed methods improve efficiency over existing non-autoregressive solvers by large margins, with AVIS achieving superior restoration performance and AVIS \emph{Flash} providing a dramatic boost in throughput while maintaining competitive restoration quality.}
We believe our framework establishes a strong foundation for future research, bridging the gap between generative priors and real-time deployment of video restoration.

\noindent\textbf{Limitations and Broader Impacts.}
While AVIS paves the way for the practical, real-time deployment of diffusion-based video inverse problem solvers, it still requires multiple VAE passes for the initial pixel-space guidance.
Incorporating measurement consistency updates directly in the latent space with auxiliary latent operators, similar to \cite{raphaeli2025silo, hong2025inversecrafter}, represents an interesting direction toward further acceleration.
Furthermore, the capability to restore high-fidelity videos from severely degraded inputs raises potential ethical concerns regarding misinformation and the unauthorized reconstruction of sensitive visual data.

\clearpage

\section*{Acknowledgments}
This work was supported by the National Research Foundation of Korea (NRF) grants funded by the Korea government (MSIT) (RS-2024-00336454, RS-2026-25468886), by the Artificial Intelligence Industrial Convergence Cluster Development Project funded by the Ministry of Science and ICT (MSIT, Korea) \& Gwangju Metropolitan City, and by a Korea University Grant (K2626081).

{
\small
\bibliographystyle{unsrtnat}
\bibliography{main}
}
 
\clearpage
\appendix
\section{Discussion of the Error Bound and Proof}
\label{appendix:proof}

Let $t_0 > t_1 > \cdots > t_K = 0$ be the reverse sampling schedule of $K$ steps for the $n$-th chunk, where $\vz^{n,\text{tgt}}$ denotes the target chunk and $\vzinit^n=\mathcal{E}(\vxinit^n)$ denotes the initial chunk encoded from the measurement-consistent estimate.
We define the initial error as:
\begin{equation}
    \epsilon_0^n := \|\vzinit^n - \vz^{n,\text{tgt}}\|.
\end{equation}
Consider the reverse sampling trajectories of both the current chunk $\vz^n_t$ and the target chunk $\vz^{n,\text{tgt}}_t$, coupled via the shared Gaussian noise $\vxi_0$.
At the start time $t_0$, the noisy states are given by:
\begin{equation}
    \vz^n_{t_0}=(1-t_0)\vzinit^n + t_0 \vxi_0,
    \quad
    \vz^{n,\text{tgt}}_{t_0}=(1-t_0)\vz^{n,\text{tgt}} + t_0 \vxi_0, \quad \vxi_{0}\sim \mathcal{N}(\mathbf{0},\mathbf{I}).
\end{equation}
Therefore, the initial error at the start time $t_0$ is explicitly scaled down by the noise level:
\begin{equation}
    \label{eq:initial_scaled_error}
    \epsilon^n_{t_0} := \|\vz^n_{t_0} - \vz^{n,\text{tgt}}_{t_0}\| = (1-t_0)\epsilon_0^n.
\end{equation}

For each reverse step $t_k \to t_{k+1}$ coupled via the shared Gaussian noise $\vxi_{k+1}$, the guidance-free updates are given by:
\begin{align}
    \hat{\vz}^n_{0|t_k} &= \vz^n_{t_k} - t_k \vfield_\theta(\vz^n_{t_k}, t_k; \KV^{<n}), \\
    \hat{\vz}^{n,\text{tgt}}_{0|t_k} &= \vz^{n,\text{tgt}}_{t_k} - t_k \vfield_\theta(\vz^{n,\text{tgt}}_{t_k}, t_k; \KV^{<n, \text{tgt}}),\\
    \vz^n_{t_{k+1}} &= (1-t_{k+1})\hat{\vz}^n_{0|t_k} + t_{k+1}\boldsymbol{\xi}_{k+1},  \\ 
    \vz^{n,\text{tgt}}_{t_{k+1}} &= (1-t_{k+1})\hat{\vz}^{n,\text{tgt}}_{0|t_k} + t_{k+1}\boldsymbol{\xi}_{k+1}, \quad \vxi_{k+1}\sim \mathcal{N}(\mathbf{0},\mathbf{I}).
\end{align}

To capture the sensitivity of the vector field to previously restored chunks, we introduce the context error $\delta_n$, defined as the average mismatch over the preceding chunks ($m < n$):
\begin{equation}
    \delta_n := \frac{1}{n-1} \sum_{m=1}^{n-1}\|\vz^m_0 - \vz^{m,\text{tgt}}_0\|.
\end{equation}
This quantity represents the accumulated prefix mismatch presented to the autoregressive model.

We assume that, at each timestep $t_k$, the vector field $\vfield_\theta$ is Lipschitz continuous with respect to both the current chunk and the context condition:
\begin{align}
    \|\vfield_\theta(\vz,t_k;\vc)-\vfield_\theta(\vz',t_k;\vc)\|
    &\le L_{z,k}\|\vz-\vz'\|,
    \\
    \|\vfield_\theta(\vz,t_k;\vc)-\vfield_\theta(\vz,t_k;\vc')\|
    &\le L_{c,k}\|\vc-\vc'\|_\text{ctx}.
\end{align}

We are now ready to prove Proposition~\ref{obs:chunkwise_stability}.

\bound*

\begin{proof}
Since the trajectories for the current chunk $\vz^n_{t_{k+1}}$ and the target chunk $\vz^{n, \text{tgt}}_{t_{k+1}}$ are coupled via the shared Gaussian noise at step $k+1$, the explicit noise terms cancel in the difference:
\begin{align*}
    &\|\vz^n_{t_{k+1}} - \vz^{n,\text{tgt}}_{t_{k+1}}\| \\
    &= (1-t_{k+1}) \Big\|\vz^n_{t_k} - \vz^{n,\text{tgt}}_{t_k} - t_k\Big(\vfield_\theta(\vz^n_{t_k}, t_k; \KV^{<n})
    -\vfield_\theta(\vz^{n,\text{tgt}}_{t_k}, t_k; \KV^{<n,\text{tgt}})\Big)\Big\|.
\end{align*}
By the triangle inequality and the Lipschitz assumptions,
\begin{align*}
    \|\vz^n_{t_{k+1}} - \vz^{n,\text{tgt}}_{t_{k+1}}\| &\leq (1-t_{k+1})\Big(\|\vz^n_{t_k} - \vz^{n,\text{tgt}}_{t_k}\|
    +t_k L_{z,k}\|\vz^n_{t_k} - \vz^{n,\text{tgt}}_{t_k}\| +t_k L_{c,k}\delta_n\Big) \\
    &=(1-t_{k+1})(1+t_k L_{z,k})\|\vz^n_{t_k} - \vz^{n,\text{tgt}}_{t_k}\| +(1-t_{k+1})t_k L_{c,k}\delta_n.
\end{align*}
Let $\lambda_k := (1-t_{k+1})(1+t_k L_{z,k})$ and $\beta_k := (1-t_{k+1})t_k L_{c,k}$. This yields the recursive relation:
\begin{equation}
    \epsilon_{t_{k+1}}^n \leq \lambda_k \epsilon_{t_k}^n + \beta_k \delta_n.
\end{equation}
Applying this relation recursively over $k=0,\cdots,K-1$ yields the final error after $K$ sampling steps:
\begin{align}
    \epsilon_K^n = \epsilon_{t_K}^n & \leq \left(\prod_{r=0}^{K-1}\lambda_r\right)\epsilon_{t_0}^n
    + \sum_{r=0}^{K-1}\left(\prod_{\ell=r+1}^{K-1}\lambda_\ell\right)\beta_r\,\delta_n \\
    & \leq \underbrace{\left((1-t_0)\prod_{r=0}^{K-1}\lambda_r\right)}_{:= \Lambda_K} \epsilon_0^n + \underbrace{\sum_{r=0}^{K-1}\left(\prod_{\ell=r+1}^{K-1}\lambda_\ell\right)\beta_r}_{:= B_K} \delta_n.
\end{align}

This concludes the proof of the Proposition.
\end{proof}

\paragraph{Discussion of the Start Time $t_0$.}
For a fixed number of sampling steps $K$, the start time $t_0$ affects the two terms in the bound differently.
Under the linear schedule $t_k=(1-k/K)t_0$, the context-dependent coefficient $B_K$ decreases as $t_0$ becomes smaller, indicating that a smaller $t_0$ suppresses the propagation of prefix mismatch across chunks.
Meanwhile, the coefficient $\Lambda_K$ does not possess an interior minimizer with respect to $t_0$. Its lowest values are approached only near the boundaries ($t_0\to 0$ or $t_0\to 1$), depending on the local Lipschitz constants of the vector field.
Consequently, the bound is better viewed as an interpretive insight rather than a prescription for the optimal choice of $t_0$.
In particular, as $t_0 \to 1$, the starting latent becomes nearly pure noise, so the reverse process becomes progressively less dependent on the pre-restored initial estimate and more dominated by the learned video prior and autoregressive context.
This helps explain why a larger $t_0$ empirically leads to more generation-like behavior, whereas a smaller $t_0$ better preserves restoration fidelity in our evaluations.

\clearpage
\section{Experimental Details}
\label{appendix:details}

\subsection{Implementation Details}
\label{appendix:implementation}

Our implementation is built upon the Self Forcing codebase~\cite{huang2025self} with Wan2.1-T2V-1.3B~\cite{wan2025wan}.
To ensure straightforward reproducibility, we retain the original configurations, except for specific parameter settings adapted for video restoration tasks (initial timestep $t_0$ and sampling steps $K$).
For the video encoding, we use Wan-VAE~\cite{wan2025wan} with spatial and temporal downsampling factors of 8 and 4, respectively.
Consequently, an input RGB video with $1+T$ frames at a resolution of $H \times W$ is encoded into a latent tensor of $1+T/4$ latent frames with spatial dimensions of $H/8 \times W/8$.
Following recent advancements in autoregressive (AR) modeling~\cite{yin2025slow, huang2025self}, we generate a chunk of $L=3$ latent frames at a time, consistent with the original Self Forcing setup.
Although the original framework employs a 4-step diffusion, we found that adapting it to a 2-step process ($K=2$) with a start time of $t_0=0.1$ provides a favorable balance between computational efficiency and visual quality, specifically for video restoration.

\paragraph{Initial Estimation.}
We detail the procedure for obtaining the initial estimate through multi-step optimization. 
To derive the initial estimate for each degradation, we employ the conjugate gradient (CG) method to optimize Eq.~\eqref{eq:cg}. 
For the main experiments without measurement noise, we set $\lambda = 0$.
By default, the optimization is initialized directly with the measurement $\vy$. 
However, to address dimensional mismatches in specific tasks, we apply tailored preprocessing to $\vy$. 
For super-resolution and spatio-temporal averaging, we align the spatial dimensions using bilinear interpolation. 
For inpainting, inspired by DDVM~\cite{saxena2023surprising}, we apply nearest-neighbor infilling to the measurement to establish a robust starting point. 
Following this initialization, we perform a task-specific number of CG updates: 5 steps for Gaussian deblurring and super-resolution, 50 steps for temporal averaging, and 100 steps for spatio-temporal averaging. 
Notably, for inpainting, since the initialized measurement already satisfies the measurement consistency term, no further CG updates are required.

\paragraph{Inference and Optimization Setup.}
Once the initial estimate is obtained, it serves as the starting point for the reverse diffusion process.
For the sampling configuration, we employ a start time of $t_0=0.1$ and $K=2$ sampling steps for the main experiments.
The justification for these optimal values is provided in our ablation studies (Section~\ref{sec:ablation}).
For the DDS update, we avoid complex parameter tuning by fixing $\gamma=1$ for the proximal optimization in Eq.~\eqref{eq:proximal}.
Additionally, we consistently conduct 5 CG updates across all degradation types.

\paragraph{Computing Resources.}
Leveraging the efficiency of our framework, we process videos at a $480\times832$ resolution with a 4-second initial latency and a 5.9 FPS throughput, requiring only 18.4 GB of VRAM on a single consumer-grade GPU (NVIDIA RTX 4090).
Furthermore, we confirm that deploying our framework on a single NVIDIA H100 GPU accelerates processing even more, achieving an initial latency of $1.85$ seconds and a throughput of $10.18$ FPS with a VRAM consumption of $26.79$ GB.

\subsection{Baseline Models}
\label{appendix:baselines}

The following summarizes the key ideas of the diffusion model-based video inverse problem solvers (DVIS) used for the evaluation.

\begin{itemize}
\item \textbf{DiffIR2VR-Zero}~\cite{yeh2024diffir2vr} adapts pre-trained image restoration diffusion models~\cite{lin2024diffbir} for video restoration without additional training, maintaining temporal consistency through hierarchical latent warping and token merging techniques.
\item \textbf{VISION-XL}~\cite{kwon2025vision} is a zero-shot inverse problem solver for high-definition video restoration using latent image diffusion models~\cite{podell2024sdxl}, efficiently achieving temporal consistency on a single GPU through pseudo-batch sampling and inversion strategies.
\item \textbf{LVTINO}~\cite{spagnoletti2026lvtino} is a zero-shot inverse problem solver for high-definition video restoration that integrates Video Consistency Models (VCMs)~\cite{yin2025slow} and Image Consistency Models (ICMs)~\cite{yin2024improved}, thereby reducing sampling steps and achieving superior restoration quality.
\end{itemize}
For all baselines, we used their official implementations in the same environment as AVIS.
We configured them to process the same resolution and number of frames, while applying the same degradation operators.
Additionally, since LVTINO exceeded the 24GB VRAM limit, we pre-computed the text embeddings and removed the text encoders from both the image and video diffusion models during inference to fit within memory constraints.

\paragraph{Other Baselines.}
Although there exist other potential candidates for DVIS, direct comparisons with methods such as SVI~\cite{kwon2025solving} and VDPS~\cite{kwon2025video} were infeasible due to limitations in their supported resolutions and high memory consumption caused by heavy operations (\eg, batch processing and Jacobian computations).
Furthermore, several recent methods~\cite{chang2024how, daras2024warped, bai2025instantvir} were excluded from our evaluation due to the lack of publicly available implementations.
Nevertheless, we believe our experimental setup provides a sufficiently representative baseline for evaluating zero-shot video restoration.
As VISION-XL~\cite{kwon2025vision} and LVTINO~\cite{spagnoletti2026lvtino} are currently recognized as the leading zero-shot solvers for high-definition video restoration, our experiments successfully demonstrate the effectiveness of our proposed method within the scope of zero-shot video inverse problem methodologies.

\subsection{Evaluation Metrics}
\label{appendix:evaluation}

To ensure straightforward reproducibility, our evaluations rely on official implementations and widely adopted PyTorch~\cite{paszke2019pytorch} packages.
For frame-wise evaluation, we extract individual frames from each video sequence to compute the respective metrics.
Specifically, we utilize \texttt{pytorch-msssim} for SSIM~\cite{wang2004image}, \texttt{lpips} for LPIPS~\cite{zhang2018unreasonable}, and \texttt{pytorch-fid} for FID~\cite{heusel2017gans}.
For video quality assessment, we employ FVD~\cite{unterthiner2019fvd} and VBench~\cite{huang2024vbench}, adopting a publicly available repository~\cite{ge2024content} for FVD computation and the official repository for VBench.
All primary evaluations are conducted on a single NVIDIA RTX 4090 GPU.
To further demonstrate the computational efficiency of our method, we additionally report inference times measured on a single NVIDIA H100 GPU.

For completeness, we provide a concise description of the evaluation metrics used in our experiments:

\begin{itemize}
    \item \textbf{PSNR} measures the pixel-level accuracy between the restored and ground-truth frames.
    \item \textbf{SSIM} evaluates the structural similarity between the restored and ground-truth frames.
    \item \textbf{LPIPS} calculates the perceptual similarity based on feature embeddings extracted from a pre-trained backbone~\cite{zhang2018unreasonable}.
    \item \textbf{FID} measures the distance between the feature distributions of generated and real images using the Inception-v3 network~\cite{szegedy2016rethinking}.
    It is computed under the assumption that both feature distributions follow multivariate Gaussian distributions.
    \item \textbf{FVD} extends the principles of FID to measure the quality and diversity of generated videos, utilizing features extracted from an Inflated 3D ConvNet (I3D)~\cite{carreira2017quo}.
    \item \textbf{Subject Consistency (VBench)} evaluates whether the appearance of a subject remains consistent by calculating the DINO~\cite{caron2021emerging} feature similarity across frames.
    \item \textbf{Background Consistency (VBench)} assesses the temporal stability of the background by computing the CLIP~\cite{radford2021learning} feature similarity across frames.
    \item \textbf{Motion Smoothness (VBench)} measures the temporal continuity of the generated motions using motion priors derived from a video frame interpolation model~\cite{li2023amt}.
    \item \textbf{Aesthetic Quality (VBench)} evaluates the overall artistic appeal and visual naturalness of the video frames using the LAION aesthetic predictor~\cite{laion2022}.
    \item \textbf{Imaging Quality (VBench)} evaluates the visual distortion (\eg, over-exposure, noise, blur) appearing in the video frames using the MUSIQ~\cite{ke2021musiq} image quality predictor.
\end{itemize}

\clearpage

\section{Further Experimental Results}
\label{appendix:more_experiments}

\subsection{Long Video Restoration}
\label{appendix:long_video}

\begin{figure}[h]
    \centering
    \includegraphics[width=\linewidth]{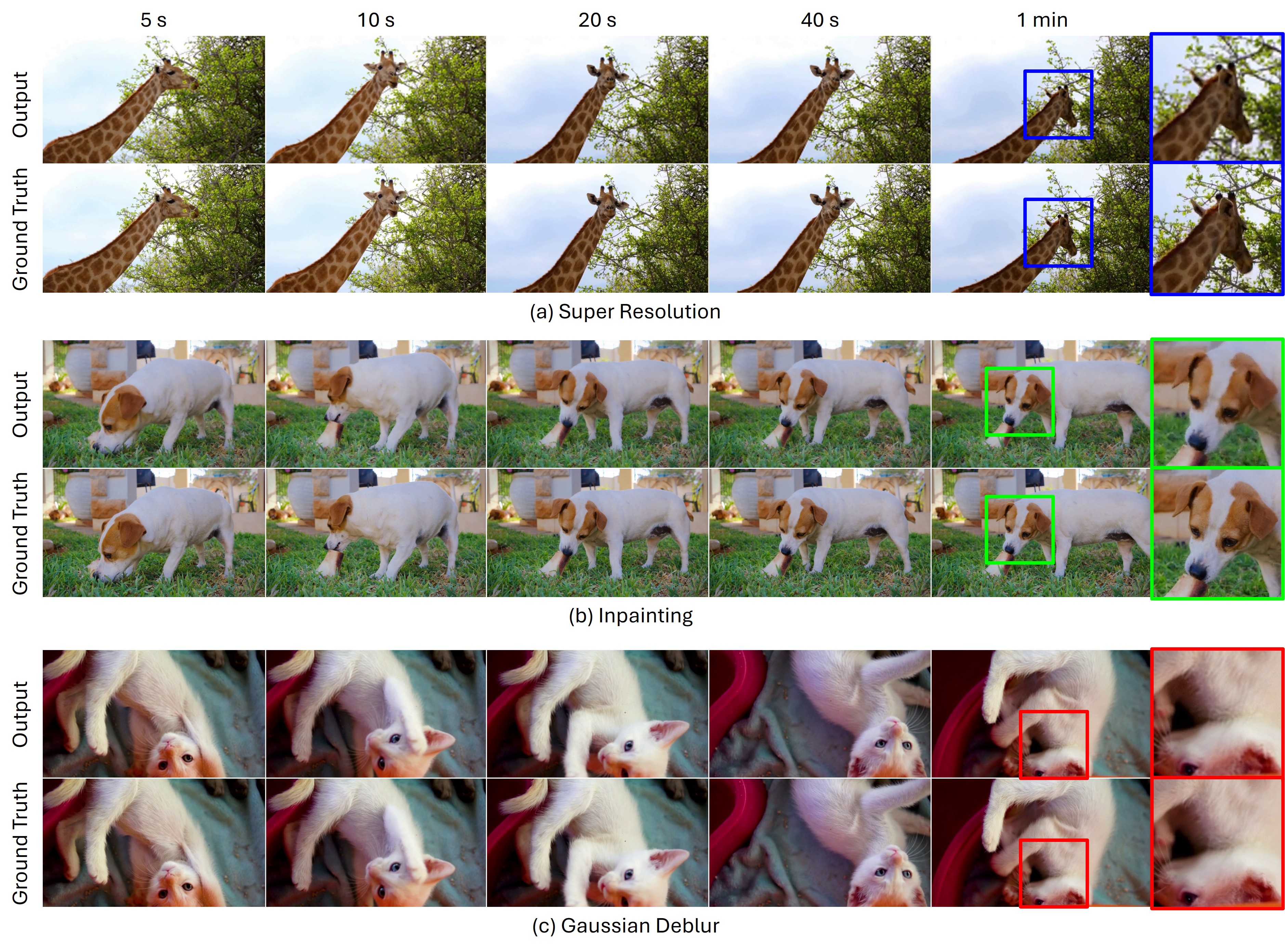}
    \vspace{-2.0em}
    \caption{\textbf{Qualitative results of long video restoration.} By periodically re-injecting measurement consistency (every 7 chunks), AVIS \emph{Flash} effectively prevents temporal drift and remains consistent with the ground truth, even in the final frames. The timestamps at the top indicate the elapsed time within the long video.}
    \label{fig:long_video}
\end{figure}

To further support the claims made in Section~\ref{sec:avis_flash}, we conduct additional experiments evaluating the scalability of our approach on longer video sequences.
Specifically, we aim to restore a one-minute video consisting of 960 frames at 16 FPS.
As described in Section~\ref{sec:avis_flash}, AVIS \emph{Flash} can be further stabilized by periodically re-injecting measurement consistency.
In this long video experiment, we apply guidance every 7 chunks to prevent error accumulation over time.

As demonstrated in Figure~\ref{fig:long_video}, even for a one-minute sequence, AVIS \emph{Flash} successfully retains the ground truth context and prevents noticeable temporal drift in later frames.
This periodic resetting effectively mitigates accumulated context mismatch, ensuring robust restoration regardless of the overall sequence length.

\subsection{Robustness of AVIS \emph{Flash} on Long Videos}
\label{appendix:robustness}
We analyze the robustness of AVIS \emph{Flash} on long videos from the Pexels dataset with respect to video length and guidance interval.

\noindent\textbf{Video length.}
We evaluate AVIS \emph{Flash} on 30-second (480-frame) and 60-second (960-frame) videos across all degradations (Table~\ref{tab:duration}).
Restoration quality is preserved as the sequence length doubles, with no sign of degradation as the autoregressive rollout proceeds.

\begin{table}[h]
    \centering
    \small
    \setlength{\abovecaptionskip}{4pt}
    \resizebox{0.5\linewidth}{!}{%
    \begin{tabular}{l | ccccc}
    \toprule
    Duration & PSNR$\uparrow$ & SSIM$\uparrow$ & LPIPS$\downarrow$ & FID$\downarrow$ & FVD$\downarrow$ \\
    \midrule
    \multicolumn{6}{l}{\textbf{\emph{Super Resolution}}} \\
    \quad 30 s & \textbf{30.49} & 0.848 & \textbf{0.235} & 20.19 & 141.5 \\
    \quad 60 s & 30.45 & \textbf{0.850} & \textbf{0.235} & \textbf{16.80} & \textbf{124.1} \\
    \midrule
    \multicolumn{6}{l}{\textbf{\emph{Inpainting}}} \\
    \quad 30 s & \textbf{30.51} & \textbf{0.853} & \textbf{0.273} & 15.06 & 158.4 \\
    \quad 60 s & 30.44 & \textbf{0.853} & 0.274 & \textbf{11.88} & \textbf{150.1} \\
    \midrule
    \multicolumn{6}{l}{\textbf{\emph{Gaussian Deblur}}} \\
    \quad 30 s & \textbf{30.82} & 0.850 & \textbf{0.273} & 17.62 & 115.5 \\
    \quad 60 s & 30.79 & \textbf{0.851} & \textbf{0.273} & \textbf{14.95} & \textbf{82.50} \\
    \midrule
    \multicolumn{6}{l}{\textbf{\emph{Temporal Average}}} \\
    \quad 30 s & \textbf{32.14} & \textbf{0.894} & \textbf{0.127} & 6.480 & 103.2 \\
    \quad 60 s & 32.02 & 0.893 & 0.128 & \textbf{4.370} & \textbf{96.53} \\
    \midrule
    \multicolumn{6}{l}{\textbf{\emph{Spatio-Temporal Average}}} \\
    \quad 30 s & \textbf{30.29} & 0.843 & \textbf{0.248} & 21.14 & 133.1 \\
    \quad 60 s & 30.25 & \textbf{0.845} & \textbf{0.248} & \textbf{17.51} & \textbf{127.7} \\
    \bottomrule
    \end{tabular}
    }
    \caption{%
        \textbf{AVIS \emph{Flash} performance across video durations.} Restoration quality is preserved even when the video length doubles.
    }
    \vspace{-0.5em}
    \label{tab:duration}
\end{table}

\noindent\textbf{Guidance interval.}
We evaluate the effect of the guidance interval $N$, where guidance is applied every $N$ chunks, on 60-second videos (Table~\ref{tab:interval}).
Dense guidance yields better quality, while sparse guidance yields larger speedup.
Since most of the speedup is obtained by increasing $N$ from 1 to 7, we adopt $N{=}7$ in AVIS \emph{Flash}.

\begin{table}[h]
\centering
\small
\setlength{\abovecaptionskip}{4pt}
\resizebox{0.6\linewidth}{!}{%
\begin{tabular}{l | ccccc | c}
\toprule
Interval $N$ & PSNR$\uparrow$ & SSIM$\uparrow$ & LPIPS$\downarrow$ & FID$\downarrow$ & FVD$\downarrow$ & Speedup \\
\midrule
\multicolumn{7}{l}{\textbf{\emph{Super Resolution}}} \\
\quad 1  & \textbf{30.92} & \textbf{0.856} & \textbf{0.223} & \textbf{15.98} & \textbf{77.60} & 1.00$\times$ \\
\quad 7  & 30.45 & 0.850 & 0.235 & 16.80 & 124.1 & 4.94$\times$ \\
\quad 14 & 30.40 & 0.849 & 0.237 & 17.14 & 120.6 & 5.35$\times$ \\
\midrule
\multicolumn{7}{l}{\textbf{\emph{Inpainting}}} \\
\quad 1  & \textbf{31.43} & \textbf{0.877} & \textbf{0.227} & \textbf{9.140} & \textbf{140.3} & 1.00$\times$ \\
\quad 7  & 30.44 & 0.853 & 0.274 & 11.88 & 150.1 & 3.54$\times$ \\
\quad 14 & 30.31 & 0.849 & 0.281 & 12.49 & 160.9 & 3.81$\times$ \\
\midrule
\multicolumn{7}{l}{\textbf{\emph{Gaussian Deblur}}} \\
\quad 1  & \textbf{31.17} & \textbf{0.854} & \textbf{0.265} & \textbf{14.47} & \textbf{73.93} & 1.00$\times$ \\
\quad 7  & 30.79 & 0.851 & 0.273 & 14.95 & 82.50 & 3.83$\times$ \\
\quad 14 & 30.74 & 0.850 & 0.275 & 15.28 & 93.31 & 5.80$\times$ \\
\midrule
\multicolumn{7}{l}{\textbf{\emph{Temporal Average}}} \\
\quad 1  & \textbf{32.80} & \textbf{0.905} & \textbf{0.116} & \textbf{4.150} & \textbf{77.40} & 1.00$\times$ \\
\quad 7  & 32.02 & 0.893 & 0.128 & 4.370 & 96.53 & 4.93$\times$ \\
\quad 14 & 31.89 & 0.890 & 0.130 & 4.440 & 92.36 & 5.34$\times$ \\
\midrule
\multicolumn{7}{l}{\textbf{\emph{Spatio-Temporal Average}}} \\
\quad 1  & \textbf{30.52} & \textbf{0.848} & \textbf{0.241} & \textbf{16.82} & \textbf{95.20} & 1.00$\times$ \\
\quad 7  & 30.25 & 0.845 & 0.248 & 17.51 & 127.7 & 4.89$\times$ \\
\quad 14 & 30.22 & 0.845 & 0.249 & 17.71 & 117.7 & 5.28$\times$ \\
\bottomrule
\end{tabular}
}
\caption{%
    \textbf{Effect of the guidance interval on 60-second videos.} Guidance is applied every $N$ chunks. Denser guidance improves quality, while sparser guidance yields larger speedup.
}
\vspace{-0.5em}
\label{tab:interval}
\end{table}

\subsection{Novel View Synthesis}
\label{appendix:nvs}

\begin{figure}[h]
    \centering
    \includegraphics[width=0.95\linewidth]{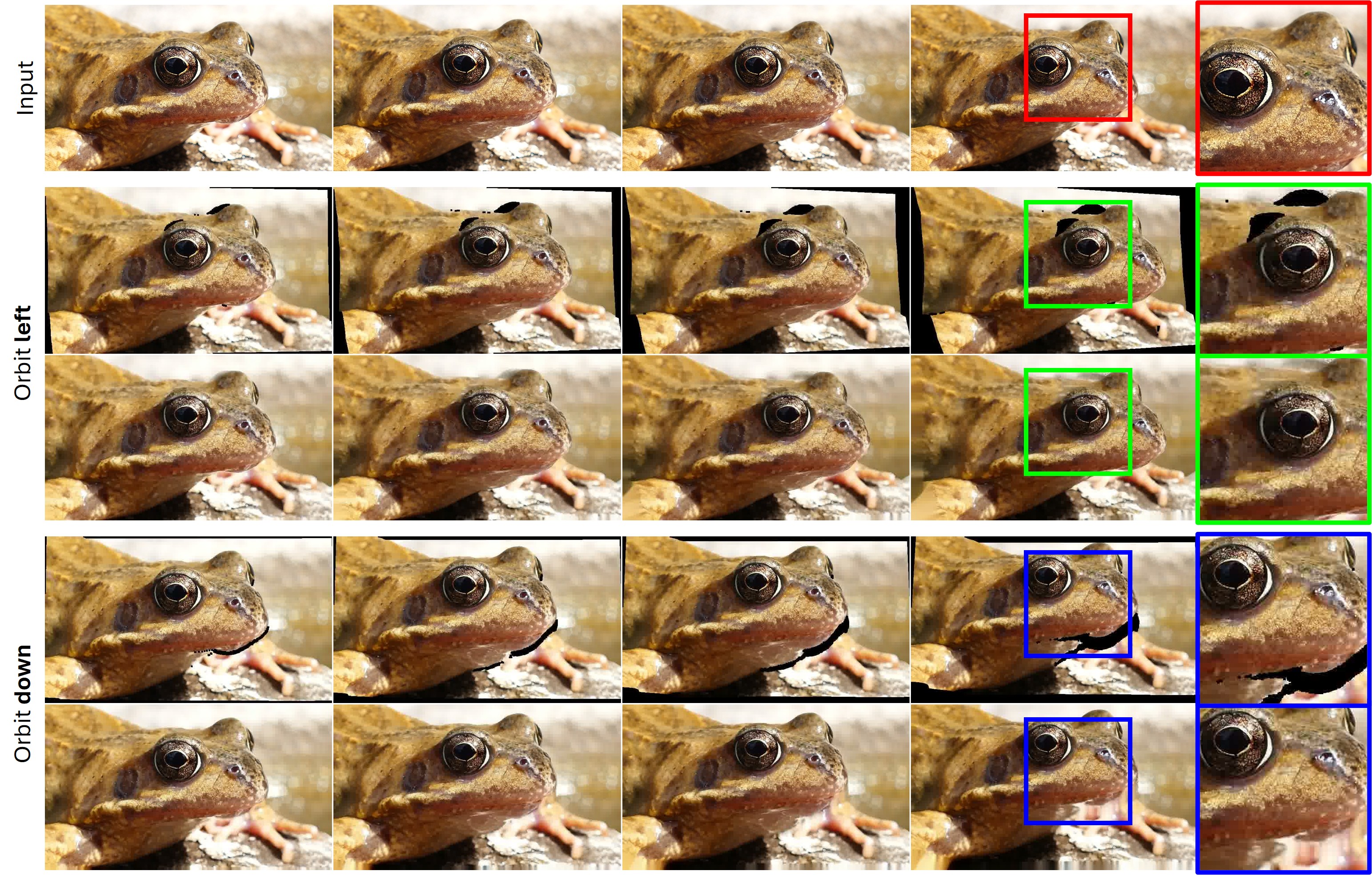}
    \vspace{-0.7em}
    \caption{{\textbf{Qualitative results of novel view video synthesis.} We demonstrate the capability of AVIS \emph{Flash} to generate novel views under target camera trajectories, including orbit left and orbit down.
    For each camera trajectory, the upper row shows the initial warped frames, and the lower row shows the inpainted outputs of AVIS \emph{Flash}.
    The boxes highlight corresponding regions in the final frames, illustrating how the method can plausibly synthesize the intended camera movements.}}
    \label{fig:nvs}
\end{figure}

In this section, we explore the feasibility of applying AVIS \emph{Flash} for inference-time novel view synthesis.
We adopt the depth-based geometric warping pipeline from~\cite{jeong2025reangle} for initial novel view rendering.
Specifically, we first estimate the depth map of each video frame using the monocular depth estimator~\cite{yang2024depth}.
Using the input video frames and their corresponding depth maps, we then lift the pixels to construct dynamic 3D point clouds.
These point clouds are then reprojected onto the 2D plane according to a target camera trajectory.

During this reprojection, disocclusions naturally occur, exposing previously occluded regions and resulting in ``holes'' in the warped frames.
We demonstrate that AVIS \emph{Flash} can act as an inpainting operator to fill in these missing regions.
As shown in Figure~\ref{fig:nvs}, AVIS \emph{Flash} can plausibly synthesize the disoccluded content, demonstrating its potential applicability for inference-time novel view synthesis.

\subsection{Spatio-Temporal Super-Resolution}
\label{appendix:stsr}
We evaluate AVIS and AVIS \emph{Flash} on noisy spatio-temporal SR with scale factors of 4$\times$ and 8$\times$ on GoPro240~\cite{nah2017deep} and Adobe240~\cite{su2017deep}.
These settings follow the official protocols of LVTINO~\cite{spagnoletti2026lvtino}.
As shown in Table~\ref{tab:stsr}, AVIS achieves the best overall restoration quality on both datasets and at both scale factors, with especially large reductions in FVD.
Notably, AVIS \emph{Flash} remains close to AVIS even at the 8$\times$ scale factor while bypassing measurement updates for subsequent chunks.
These results indicate that both variants remain effective under a severely ill-conditioned forward operator.

\begin{table*}[h]
\centering
\small
\setlength{\abovecaptionskip}{4pt}
\resizebox{0.9\textwidth}{!}{%
\begin{tabular}{l | ccccc | ccccc}
\toprule
 & \multicolumn{5}{c|}{Scale factor 4$\times$} & \multicolumn{5}{c}{Scale factor 8$\times$} \\
\cmidrule(lr){2-6} \cmidrule(l){7-11}
Method & PSNR$\uparrow$ & SSIM$\uparrow$ & LPIPS$\downarrow$ & FID$\downarrow$ & FVD$\downarrow$ & PSNR$\uparrow$ & SSIM$\uparrow$ & LPIPS$\downarrow$ & FID$\downarrow$ & FVD$\downarrow$ \\
\midrule
\multicolumn{11}{l}{\textbf{\emph{GoPro240}}} \\
\quad LVTINO        & 26.45 & 0.761 & 0.175 & 65.51 & 474.3 & 22.69 & 0.620 & 0.354 & \textbf{138.6} & 923.2 \\
\rowcolor{custompink!20}
\quad AVIS          & \textbf{27.29} & \textbf{0.814} & \textbf{0.135} & \textbf{56.94} & \textbf{124.2} & \textbf{23.21} & \textbf{0.646} & \textbf{0.317} & 141.8 & \textbf{542.2} \\
\rowcolor{customblue!30}
\quad AVIS \emph{Flash} & \underline{27.02} & \underline{0.805} & \underline{0.145} & \underline{59.77} & \underline{141.2} & \underline{23.09} & \underline{0.642} & \underline{0.323} & \underline{139.4} & \underline{559.5} \\
\midrule
\multicolumn{11}{l}{\textbf{\emph{Adobe240}}} \\
\quad LVTINO        & 25.38 & 0.745 & 0.183 & 55.84 & 593.1 & 21.93 & \underline{0.625} & 0.362 & 125.0 & 833.3 \\
\rowcolor{custompink!20}
\quad AVIS          & \textbf{26.32} & \textbf{0.777} & \textbf{0.166} & \textbf{50.39} & \textbf{167.1} & \textbf{22.59} & \textbf{0.629} & \textbf{0.343} & \underline{113.4} & \textbf{570.7} \\
\rowcolor{customblue!30}
\quad AVIS \emph{Flash} & \underline{26.04} & \underline{0.767} & \underline{0.174} & \underline{51.54} & \underline{179.9} & \underline{22.39} & 0.624 & \underline{0.348} & \textbf{111.8} & \underline{574.8} \\
\bottomrule
\end{tabular}
}
\caption{%
    \textbf{Spatio-temporal super-resolution on GoPro240 and Adobe240.} \textcolor{custompinktext}{AVIS} and \textcolor{custombluetext}{AVIS \emph{Flash}} remain effective even at a scale factor of 8$\times$. \textbf{Bold} and \underline{underline} indicate the best and second-best results, respectively.
}
\vspace{-0.5em}
\label{tab:stsr}
\end{table*}

\subsection{Chunk-wise Measurement Consistency}
\label{appendix:residual}
We report the RMSE of the measurement consistency term ($\|\Ac(\vx_0) - \vy\|$) per chunk for both AVIS and AVIS \emph{Flash} across all degradations in Table~\ref{tab:residual}.
For AVIS, the residual stays comparable to or below that of the first chunk throughout the rollout.
For AVIS \emph{Flash}, the residual rises above the first chunk but remains nearly flat in all subsequent chunks rather than growing.
This result indicates that error accumulation in AVIS \emph{Flash} does not diverge over the autoregressive rollout.

\begin{table}[h]
\centering
\small
\setlength{\abovecaptionskip}{4pt}
\resizebox{0.7\linewidth}{!}{%
\begin{tabular}{l | ccccccc}
\toprule
Method & Chunk 1 & Chunk 2 & Chunk 3 & Chunk 4 & Chunk 5 & Chunk 6 & Chunk 7 \\
\midrule
\multicolumn{8}{l}{\textbf{\emph{Super Resolution}}} \\
\quad AVIS              & 0.030 & 0.030 & 0.030 & 0.030 & 0.030 & 0.030 & 0.030 \\
\quad AVIS \emph{Flash} & 0.030 & 0.036 & 0.036 & 0.035 & 0.036 & 0.036 & 0.036 \\
\midrule
\multicolumn{8}{l}{\textbf{\emph{Inpainting}}} \\
\quad AVIS              & 0.048 & 0.049 & 0.050 & 0.049 & 0.050 & 0.050 & 0.050 \\
\quad AVIS \emph{Flash} & 0.048 & 0.056 & 0.057 & 0.057 & 0.058 & 0.058 & 0.058 \\
\midrule
\multicolumn{8}{l}{\textbf{\emph{Gaussian Deblur}}} \\
\quad AVIS              & 0.017 & 0.017 & 0.017 & 0.016 & 0.016 & 0.016 & 0.016 \\
\quad AVIS \emph{Flash} & 0.017 & 0.021 & 0.019 & 0.019 & 0.019 & 0.020 & 0.020 \\
\midrule
\multicolumn{8}{l}{\textbf{\emph{Temporal Average}}} \\
\quad AVIS              & 0.040 & 0.036 & 0.036 & 0.037 & 0.036 & 0.037 & 0.037 \\
\quad AVIS \emph{Flash} & 0.040 & 0.041 & 0.042 & 0.043 & 0.043 & 0.042 & 0.043 \\
\midrule
\multicolumn{8}{l}{\textbf{\emph{Spatio-Temporal Average}}} \\
\quad AVIS              & 0.022 & 0.023 & 0.023 & 0.023 & 0.023 & 0.023 & 0.023 \\
\quad AVIS \emph{Flash} & 0.022 & 0.028 & 0.026 & 0.026 & 0.026 & 0.026 & 0.026 \\
\bottomrule
\end{tabular}
}
\caption{%
    \textbf{Chunk-wise measurement residual (RMSE).} The residual of AVIS \emph{Flash} remains nearly flat after the first chunk, indicating no divergent error accumulation.
}
\vspace{-0.5em}
\label{tab:residual}
\end{table}

\subsection{Backbone Fairness Check}
\label{appendix:matched_backbone}
To examine whether the gains of AVIS are primarily due to the video prior rather than the solving framework, we additionally evaluate the leading non-autoregressive solver, LVTINO~\cite{spagnoletti2026lvtino}, coupled with the same autoregressive backbone~\cite{huang2025self} as AVIS, denoted as LVTINO-AR, across all five tasks.

As shown in Table~\ref{tab:appendix}, AVIS and LVTINO-AR are broadly comparable in restoration quality.
At comparable quality, however, the two differ substantially in cost, as shown in Table~\ref{tab:backbone_runtime}.
On 4$\times$ super-resolution with 81-frame input on a single RTX 4090, LVTINO-AR requires 89.2 seconds per video (0.91 FPS), whereas AVIS requires 68.5 seconds (1.18 FPS) and AVIS \emph{Flash} 13.7 seconds (5.91 FPS).
In addition, LVTINO-AR requires an auxiliary image consistency model~\cite{yin2024improved} alongside the video prior, incurring additional memory overhead, whereas AVIS operates with a single backbone.

\begin{table}[h]
\centering
\small
\setlength{\abovecaptionskip}{4pt}
\resizebox{0.5\linewidth}{!}{%
\begin{tabular}{l | ccccc}
\toprule
Method & PSNR$\uparrow$ & SSIM$\uparrow$ & LPIPS$\downarrow$ & FVD$\downarrow$ & FID$\downarrow$ \\
\midrule
\multicolumn{6}{l}{\textbf{\emph{Super Resolution}}} \\
\quad LVTINO    & 30.04 & \underline{0.824} & \underline{0.102} & 59.42 & \textbf{18.32} \\
\quad LVTINO-AR & \underline{30.28} & 0.818 & 0.113 & \underline{48.50} & \underline{21.99} \\
\rowcolor{gray!25}
\quad AVIS      & \textbf{30.38} & \textbf{0.826} & \textbf{0.101} & \textbf{40.36} & 23.94 \\
\midrule
\multicolumn{6}{l}{\textbf{\emph{Inpainting}}} \\
\quad LVTINO    & 26.36 & 0.737 & 0.187 & 322.0 & 35.75 \\
\quad LVTINO-AR & \underline{29.79} & \underline{0.840} & \underline{0.110} & \underline{111.0} & \underline{20.59} \\
\rowcolor{gray!25}
\quad AVIS      & \textbf{31.27} & \textbf{0.870} & \textbf{0.075} & \textbf{68.90} & \textbf{10.21} \\
\midrule
\multicolumn{6}{l}{\textbf{\emph{Gaussian Deblur}}} \\
\quad LVTINO    & 30.64 & \underline{0.835} & \underline{0.109} & 47.45 & \textbf{20.37} \\
\quad LVTINO-AR & \textbf{31.53} & \textbf{0.841} & \textbf{0.100} & \textbf{27.25} & \underline{21.62} \\
\rowcolor{gray!25}
\quad AVIS      & \underline{30.69} & 0.831 & 0.110 & \underline{36.07} & 24.40 \\
\midrule
\multicolumn{6}{l}{\textbf{\emph{Temporal Average}}} \\
\quad LVTINO    & 31.68 & 0.878 & \underline{0.069} & 77.76 & 8.015 \\
\quad LVTINO-AR & \underline{32.06} & \textbf{0.893} & 0.083 & \textbf{48.31} & \textbf{5.836} \\
\rowcolor{gray!25}
\quad AVIS      & \textbf{32.10} & \underline{0.887} & \textbf{0.063} & \underline{54.85} & \underline{6.803} \\
\midrule
\multicolumn{6}{l}{\textbf{\emph{Spatio-Temporal Average}}} \\
\quad LVTINO    & 29.85 & \underline{0.818} & \underline{0.110} & 95.35 & \textbf{19.73} \\
\quad LVTINO-AR & \textbf{30.33} & \textbf{0.824} & \textbf{0.103} & \underline{78.67} & \underline{19.97} \\
\rowcolor{gray!25}
\quad AVIS      & \underline{29.99} & \underline{0.818} & 0.112 & \textbf{62.11} & 25.61 \\
\bottomrule
\end{tabular}
}
\caption{%
    \textbf{Backbone fairness check across five tasks.} LVTINO-AR denotes LVTINO coupled with the same autoregressive backbone as AVIS. \textbf{Bold} and \underline{underline} indicate the best and second-best results, respectively.
}
\vspace{-1.0em}
\label{tab:appendix}
\end{table}

\begin{table}[h]
\centering
\small
\setlength{\abovecaptionskip}{4pt}
\resizebox{0.5\linewidth}{!}{%
\begin{tabular}{l | ccc}
\toprule
Method & Time (s)$\downarrow$ & FPS$\uparrow$ & Aux. image model \\
\midrule
LVTINO    & 114  & 0.71 & Required \\
LVTINO-AR & 89.2 & 0.91 & Required \\
\rowcolor{custompink!20}
AVIS      & \underline{68.5} & \underline{1.18} & Not required \\
\rowcolor{customblue!30}
AVIS \emph{Flash} & \textbf{13.7} & \textbf{5.91} & Not required \\
\bottomrule
\end{tabular}
}
\caption{%
    \textbf{Runtime comparison on 4$\times$ super-resolution} using a single RTX 4090.
}
\vspace{-0.5em}
\label{tab:backbone_runtime}
\end{table}

\subsection{Runtime Analysis}
\label{appendix:runtime}
Table~\ref{tab:runtime} reports detailed runtime, throughput, and peak memory of AVIS \emph{Flash} for each degradation, measured on 81-frame videos using a single RTX 4090.
The total runtime is decomposed into initialization, sampling, and VAE decoding.
AVIS \emph{Flash} restores the first chunk within 4 to 7 seconds and sustains 4.9 to 5.9 FPS thereafter, with peak memory below 20 GB.
Across all degradations, the latency and throughput remain far better than those of the baselines in Table~\ref{tab:efficiency}.

\begin{table}[h]
\centering
\small
\setlength{\abovecaptionskip}{4pt}
\resizebox{\linewidth}{!}{%
\begin{tabular}{l | ccc | ccc}
\toprule
Task & Initialization (s) & Sampling (s) & VAE Decoding (s) & Latency (s) & FPS (frame/s) & Memory (GB) \\
\midrule
Super Resolution        & 2.47 & 7.13 & 4.10 & 4.07 & 5.9 & 18.4 \\
Inpainting              & 5.16 & 7.28 & 4.10 & 6.79 & 4.9 & 18.4 \\
Gaussian Deblur         & 3.31 & 7.50 & 4.10 & 4.97 & 5.4 & 17.8 \\
Temporal Average        & 3.57 & 7.26 & 4.10 & 5.19 & 5.4 & 19.9 \\
Spatio-Temporal Average & 3.87 & 7.11 & 4.10 & 5.47 & 5.4 & 18.4 \\
\bottomrule
\end{tabular}
}
\vspace{-0.3em}
\caption{%
    \textbf{Detailed runtime of AVIS \emph{Flash} for each degradation.} Measured on 81-frame videos using a single RTX 4090.
}
\vspace{-0.5em}
\label{tab:runtime}
\end{table}

\subsection{Further Results on the Effect of the Initialization}
\label{appendix:more_ablation}
We provide further results of the initialization ablation in Section~\ref{sec:ablation} on the other four tasks under the strong noise setting ($\sigma_n{=}0.05$, $\lambda{=}0.3$).
As shown in Table~\ref{tab:init_full}, AVIS achieves the best or tied-best results on every metric across all tasks, consistently improving over all three controls.
These results demonstrate that the final restoration cannot be explained by the initialization or its VAE reconstruction alone.

\begin{table}[h]
\centering
\small
\setlength{\abovecaptionskip}{4pt}
\resizebox{0.5\linewidth}{!}{%
\begin{tabular}{l | ccccc}
\toprule
Setting & PSNR$\uparrow$ & SSIM$\uparrow$ & LPIPS$\downarrow$ & FID$\downarrow$ & FVD$\downarrow$ \\
\midrule
\multicolumn{6}{l}{\textbf{\emph{Super Resolution}}} \\
\quad $\vxinit^{\lambda}$                          & 26.69 & 0.666 & 0.297 & 66.89 & 1066 \\
\quad $\mathcal{D}(\mathcal{E}(\vxinit^{\lambda}))$ & \underline{27.03} & 0.688 & \underline{0.289} & 65.27 & 1027 \\
\quad AVIS w/o DDS                                  & 27.02 & \underline{0.690} & \textbf{0.271} & \underline{59.74} & \underline{831.6} \\
\rowcolor{gray!25}
\quad AVIS                                          & \textbf{27.07} & \textbf{0.696} & \textbf{0.271} & \textbf{58.72} & \textbf{725.8} \\
\midrule
\multicolumn{6}{l}{\textbf{\emph{Gaussian Deblur}}} \\
\quad $\vxinit^{\lambda}$                          & 21.62 & 0.351 & 0.464 & 112.7 & 1585 \\
\quad $\mathcal{D}(\mathcal{E}(\vxinit^{\lambda}))$ & 22.28 & \underline{0.386} & \underline{0.454} & \underline{110.8} & 1254 \\
\quad AVIS w/o DDS                                  & \underline{22.31} & 0.382 & 0.460 & 111.4 & \underline{1236} \\
\rowcolor{gray!25}
\quad AVIS                                          & \textbf{23.77} & \textbf{0.423} & \textbf{0.410} & \textbf{89.71} & \textbf{591.1} \\
\midrule
\multicolumn{6}{l}{\textbf{\emph{Temporal Average}}} \\
\quad $\vxinit^{\lambda}$                          & \underline{27.60} & 0.721 & 0.236 & 45.72 & 464.7 \\
\quad $\mathcal{D}(\mathcal{E}(\vxinit^{\lambda}))$ & 27.52 & \underline{0.724} & \underline{0.234} & \underline{43.86} & \underline{439.9} \\
\quad AVIS w/o DDS                                  & 26.74 & 0.688 & 0.249 & 46.52 & 444.2 \\
\rowcolor{gray!25}
\quad AVIS                                          & \textbf{28.35} & \textbf{0.776} & \textbf{0.197} & \textbf{35.94} & \textbf{395.9} \\
\midrule
\multicolumn{6}{l}{\textbf{\emph{Spatio-Temporal Average}}} \\
\quad $\vxinit^{\lambda}$                          & 25.92 & 0.650 & 0.332 & 77.16 & 1393 \\
\quad $\mathcal{D}(\mathcal{E}(\vxinit^{\lambda}))$ & 26.27 & 0.672 & 0.319 & 74.58 & 1328 \\
\quad AVIS w/o DDS                                  & \underline{26.31} & \underline{0.673} & \underline{0.304} & \underline{69.28} & \underline{1045} \\
\rowcolor{gray!25}
\quad AVIS                                          & \textbf{26.41} & \textbf{0.686} & \textbf{0.295} & \textbf{67.57} & \textbf{933.3} \\
\bottomrule
\end{tabular}
}
\caption{%
    \textbf{Further results on the effect of the initialization} ($\sigma_n{=}0.05$, $\lambda{=}0.3$). AVIS consistently improves over the initialization and its VAE reconstruction. \textbf{Bold} and \underline{underline} indicate the best and second-best results, respectively.
}
\vspace{-0.3em}
\label{tab:init_full}
\end{table}

\subsection{Qualitative Evaluation}
\label{appendix:qualitative}
We provide comprehensive qualitative comparisons for each degradation type, with detailed discussions provided in the respective figure captions.
All figures are best viewed zoomed in. For complete video comparisons, please refer to our project page.

\begin{figure}[h]
    \centering
    \includegraphics[width=\linewidth]{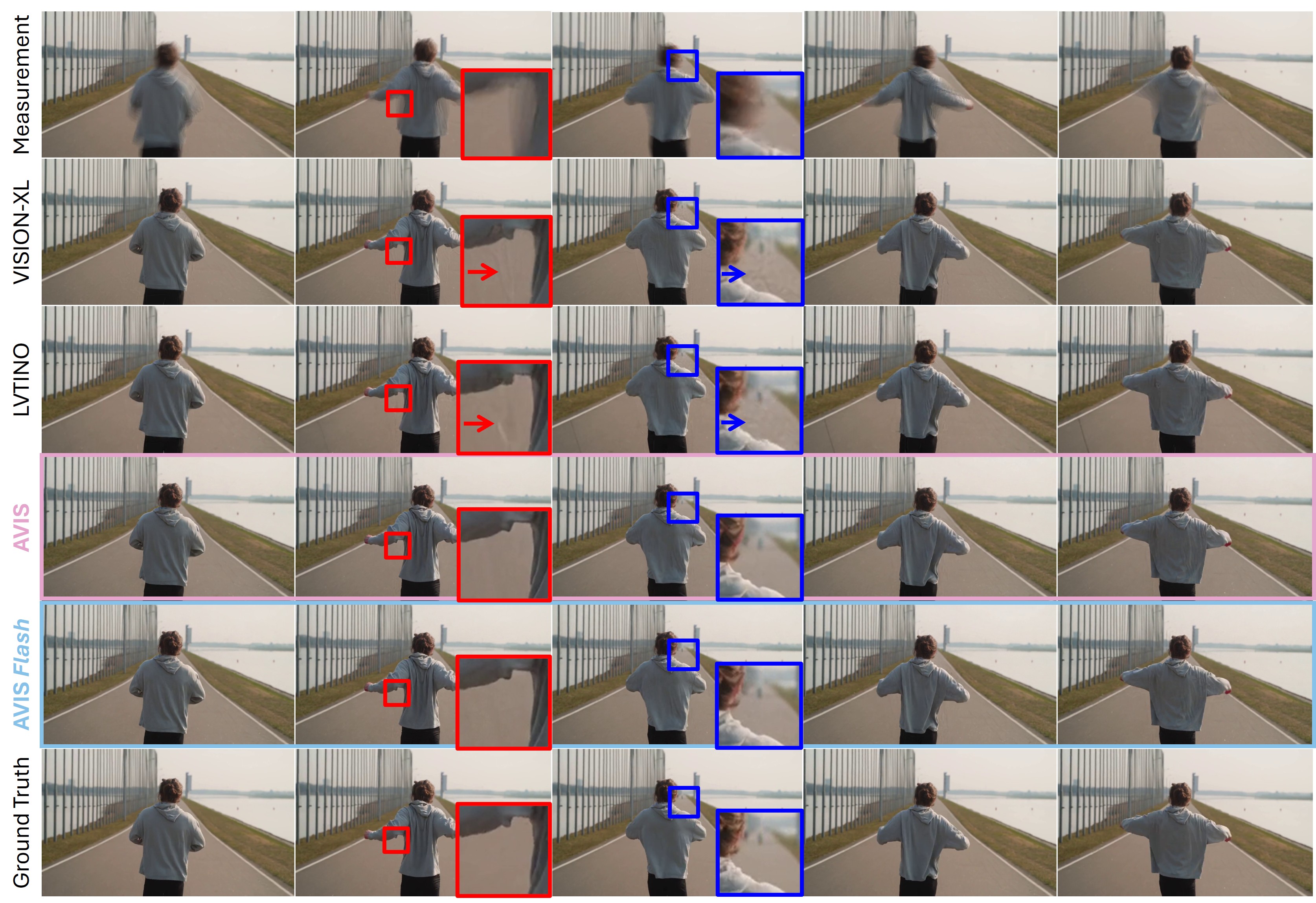}
    \vspace{-1.5em}
    \caption{\textbf{Qualitative comparisons on temporal averaging.} VISION-XL and LVTINO suffer from noticeable artifacts in the highlighted regions. In contrast, AVIS recovers the most plausible details. Notably, despite its much faster inference, AVIS \emph{Flash} maintains visual quality comparable to AVIS, preserving overall structural integrity.}
    \label{fig:qual_temporal}
\end{figure}

\begin{figure}[h]
    \centering
    \vspace{-0.5em}
    \includegraphics[width=\linewidth]{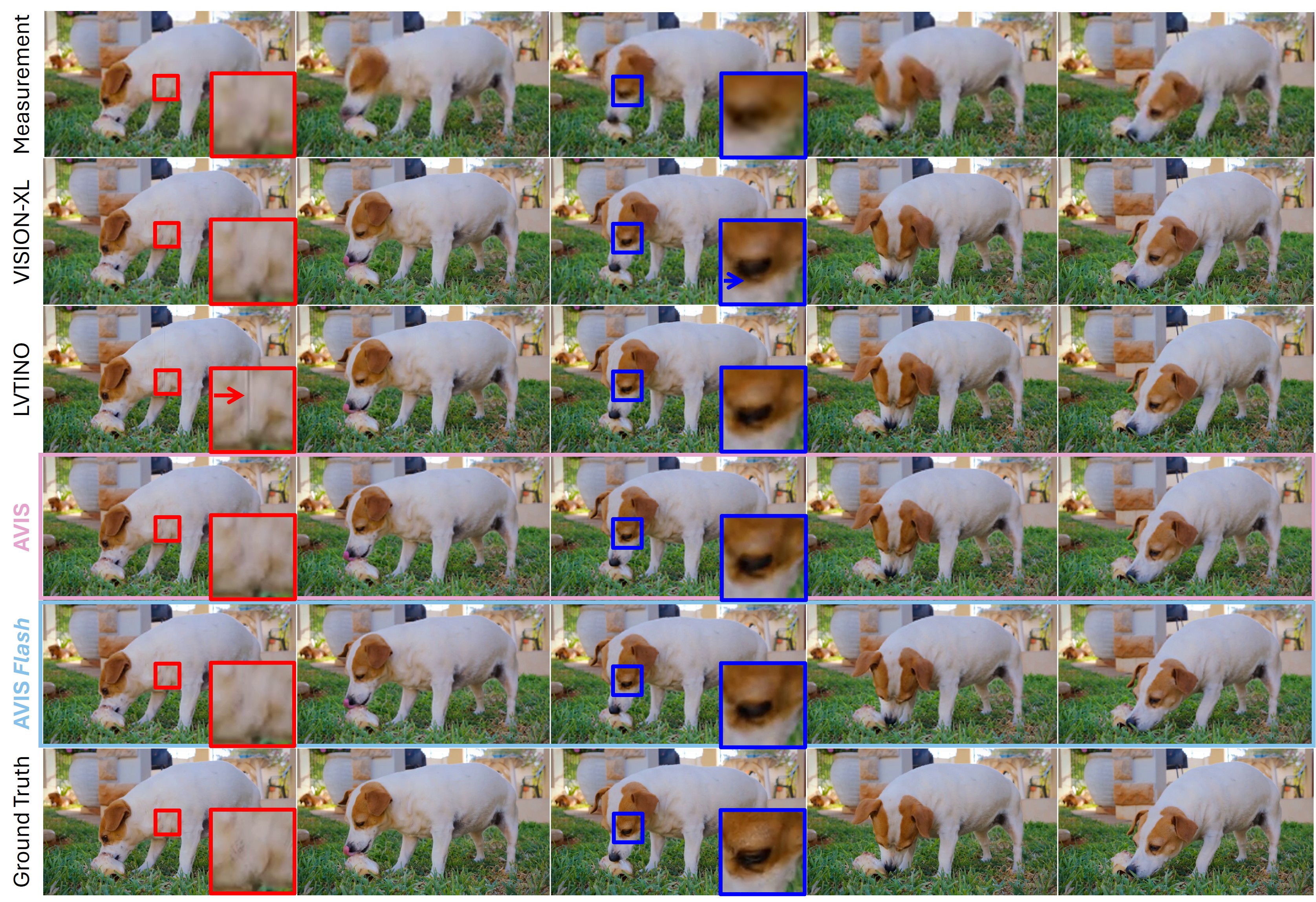}
    \vspace{-1.5em}
    \caption{\textbf{Qualitative comparisons on spatio-temporal averaging.} LVTINO produces noticeable vertical artifacts (red arrow), and VISION-XL struggles to reconstruct fine details (\eg, around the eye). In contrast, AVIS restores the most plausible details. Furthermore, despite its much higher throughput, AVIS \emph{Flash} avoids the artifacts seen in the baselines, remaining highly competitive.}
    \label{fig:qual_sptio_temporal}
\end{figure}

\begin{figure}[h]
    \centering
    \includegraphics[width=\linewidth]{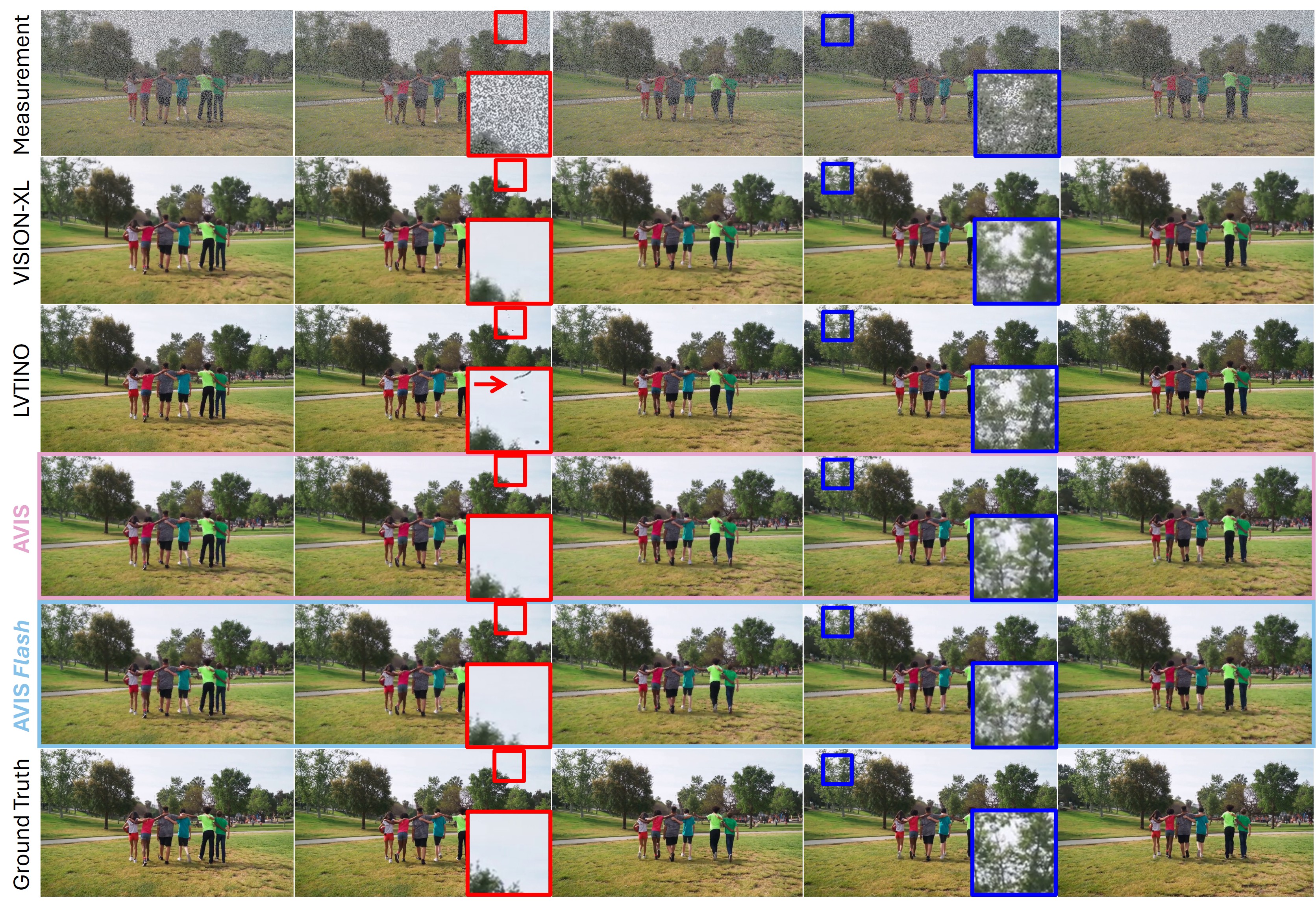}
    \vspace{-1.5em}
    \caption{\textbf{Qualitative comparisons on inpainting.} LVTINO introduces unnatural floating artifacts in the restored sky region (red arrow), and VISION-XL yields overly smoothed, blurry textures when reconstructing the trees (blue box). In contrast, AVIS and AVIS \emph{Flash} produce more plausible structures and preserve overall scene consistency.}
    \label{fig:qual_inpainting}
\end{figure}

\begin{figure}[h]
    \centering
    \vspace{-0.5em}
    \includegraphics[width=\linewidth]{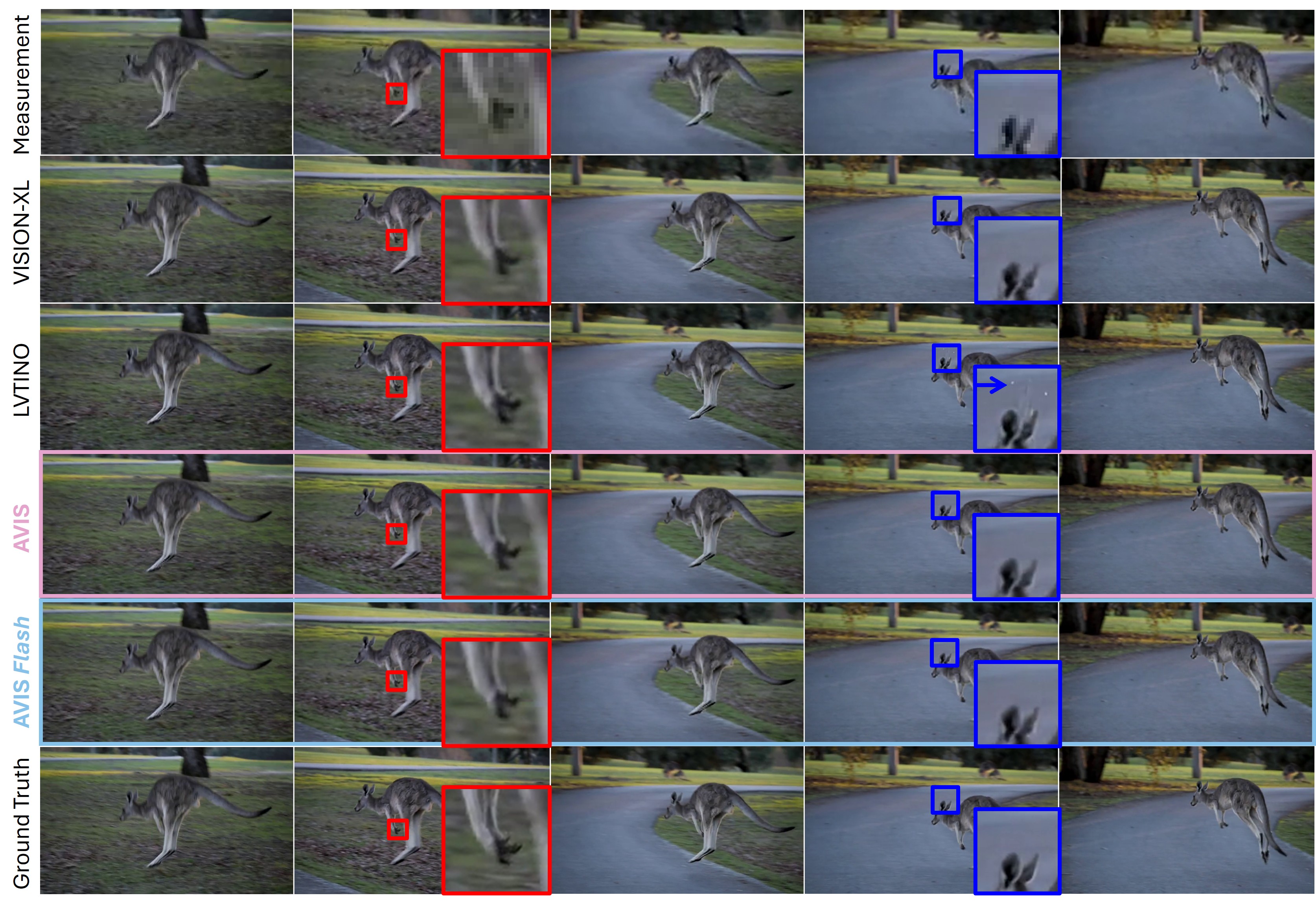}
    \vspace{-1.5em}
    \caption{\textbf{Qualitative comparisons on super-resolution.} LVTINO produces unnatural artifacts in the highlighted region (blue box). While VISION-XL and AVIS \emph{Flash} yield slightly softer results, AVIS restores finer details. Notably, even with its highly accelerated inference, AVIS \emph{Flash} successfully preserves structural integrity without severe artifacts.}
    \label{fig:qual_sr}
\end{figure}

\begin{figure}[h]
    \centering
    \includegraphics[width=\linewidth]{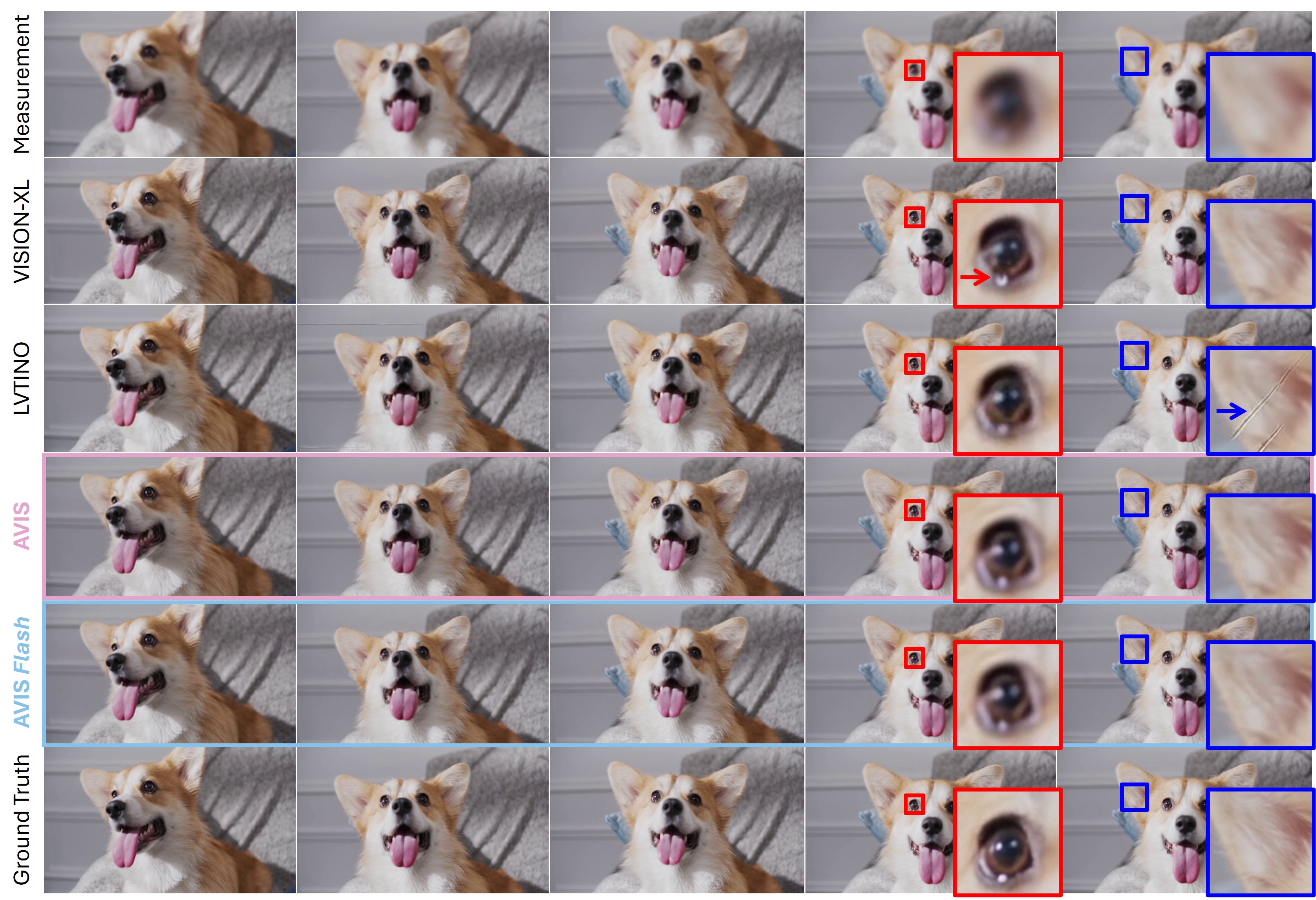}
    \vspace{-1.5em}
    \caption{\textbf{Qualitative comparisons on Gaussian deblurring.} VISION-XL introduces minor distortions (red box), and LVTINO creates unnatural artifacts (blue box). In contrast, our proposed AVIS and AVIS \emph{Flash} faithfully preserve fine details and structural integrity.}
    \label{fig:qual_gauss}
\end{figure}


\end{document}